%% file: xlingual_survey.tex
\documentclass[twoside,11pt]{article}
\usepackage{jair, theapa, rawfonts}

\jairheading{65}{2019}{569-631}{10/2017}{08/2019}
\ShortHeadings{A Survey Of Cross-lingual Word Embedding Models}
{Ruder, Vuli\'c, \& S{{\o}}gaard}
\firstpageno{569}

\usepackage[utf8]{inputenc} 
\usepackage[T1]{fontenc}    
\usepackage{url}            
\usepackage{booktabs}       
\usepackage{amsfonts}       
\usepackage{nicefrac}       
\usepackage{microtype}      
\usepackage{amsmath,amsthm}
\usepackage{graphicx}
\usepackage{subcaption}
\usepackage{url}
\usepackage{multirow, makecell}
\usepackage{todonotes}

\usepackage{listings}

\usepackage{xcolor}
\definecolor{nice-red}{HTML}{E41A1C}
\definecolor{nice-green}{HTML}{4DAF4A}
\definecolor{nice-blue}{HTML}{377EB8}

\newtheorem{lemma}{Lemma}

\begin{document}

\title{A Survey of Cross-lingual Word Embedding Models}

\author{\name Sebastian Ruder \email sebastian@ruder.io \\
       \addr Insight Research Centre, National University of Ireland,\\Galway, Ireland\\
       Aylien Ltd., Dublin, Ireland
       \AND
       \name Ivan Vulić \email iv250@cam.ac.uk \\
       \addr Language Technology Lab, University of Cambridge, UK
       \AND
       \name Anders Søgaard \email soegaard@di.ku.dk \\
       \addr University of Copenhagen, Copenhagen, Denmark}


\maketitle


\begin{abstract}
Cross-lingual representations of words enable us to reason about word meaning in multilingual contexts and are a key facilitator of cross-lingual transfer when developing natural language processing models for low-resource languages.
In this survey, we provide a comprehensive typology of cross-lingual word embedding models.
We compare their data requirements and objective functions. The recurring theme of the survey is that many of the models presented in the literature optimize for the same objectives, and that seemingly different models are often equivalent, {\em modulo}~optimization strategies, hyper-parameters, and such.
We also discuss the different ways cross-lingual word embeddings are evaluated, as well as future challenges and research horizons. 

\end{abstract}

\input{1-Intro.tex}
\input{2-Notation.tex}
\input{3-Monolingual.tex}
\input{4-Typology.tex}
\input{5-History.tex}
\input{6-1-Word-alignment-mapping.tex}

\input{6-2-Word-alignment-pseudo-bilingual.tex}
\input{6-3-Word-alignment-joint.tex}
\input{6-4-Word-alignment-comparable.tex}
\input{7-Sentence-alignment.tex}

\input{8-Document-alignment.tex}
\input{9-Multilingual.tex}

\input{10-Evaluation.tex}

\input{11-Challenges.tex}

\section{Conclusion} \label{sec:conclusion}

This survey has focused on providing an overview of cross-lingual word embedding models. It has introduced standardized notation and a typology that demonstrated the similarity of many of these models. It provided proofs that connect different word-level embedding models and has described ways to evaluate cross-lingual word embeddings as well as how to extend them to the multilingual setting. It finally outlined challenges and future directions.

\section*{Acknowledgements}

We thank the anonymous reviewers and the editors for their valuable and comprehensive feedback. Sebastian was supported by Irish Research Council Grant Number EBPPG/2014/30 and Science Foundation Ireland Grant Number SFI/12/RC/2289. Ivan's work is supported by the ERC Consolidator Grant LEXICAL: Lexical Acquisition Across Languages (no 648909). Sebastian is now affiliated with DeepMind.

\bibliography{xlingual_survey}
\bibliographystyle{theapa}

\end{document}

%% file: 1-Intro.tex
\section{Introduction}
In recent years, (monolingual) vector representations of words, so-called \textit{word embeddings} \cite{Mikolov2013a,Pennington2014} have proven extremely useful across a wide range of natural language processing (NLP) applications. In parallel, the public awareness of the digital language divide\footnote{E.g.,~\url{http://labs.theguardian.com/digital-language-divide/}}, as well as the availability of multilingual benchmarks \cite<Nivre et al., 2016a;>{hovy2006ontonotes,sylak2015language}
, has made cross-lingual transfer a popular NLP research topic. The need to transfer lexical knowledge across languages has given rise to \textit{cross-lingual word embedding models}, i.e., cross-lingual representations of words in a joint embedding space, as illustrated in Figure \ref{fig:shared_embedding_space}.

Cross-lingual word embeddings are appealing for two reasons: First, they enable us {\em to compare the meaning of words across languages}, which is key to bilingual lexicon induction, machine translation, or cross-lingual information retrieval, for example. Second, cross-lingual word embeddings {\em enable model transfer between languages}, e.g., between resource-rich and low-resource languages, by providing a common representation space. This duality is also reflected in how cross-lingual word embeddings are evaluated, as discussed in Section~\ref{sec:evaluation}.


\begin{figure}
\centering
\begin{subfigure}{.5\textwidth}
  \centering
  \includegraphics[width=.9\linewidth]{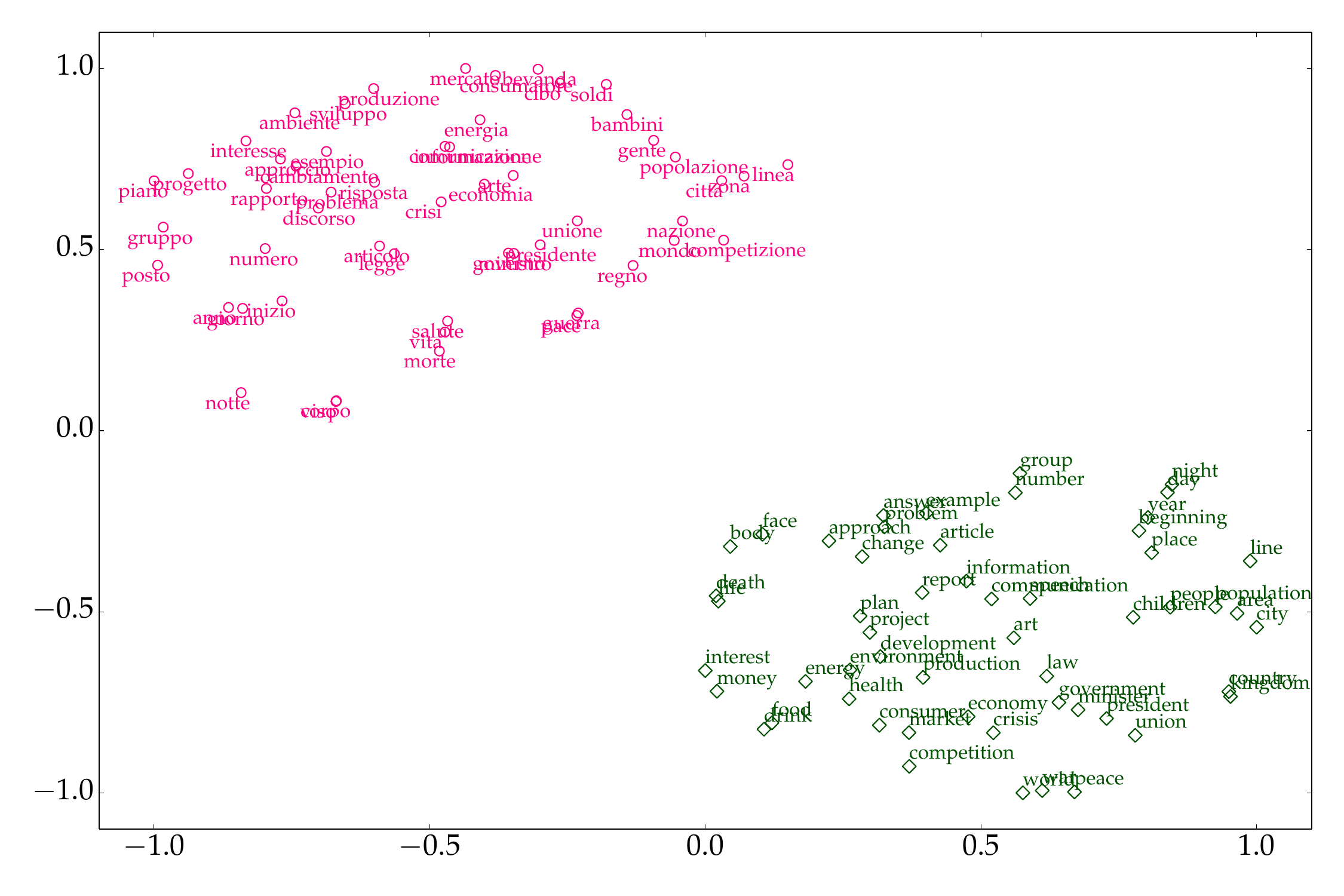}
\end{subfigure}%
\begin{subfigure}{.5\textwidth}
  \centering
  \includegraphics[width=.9\linewidth]{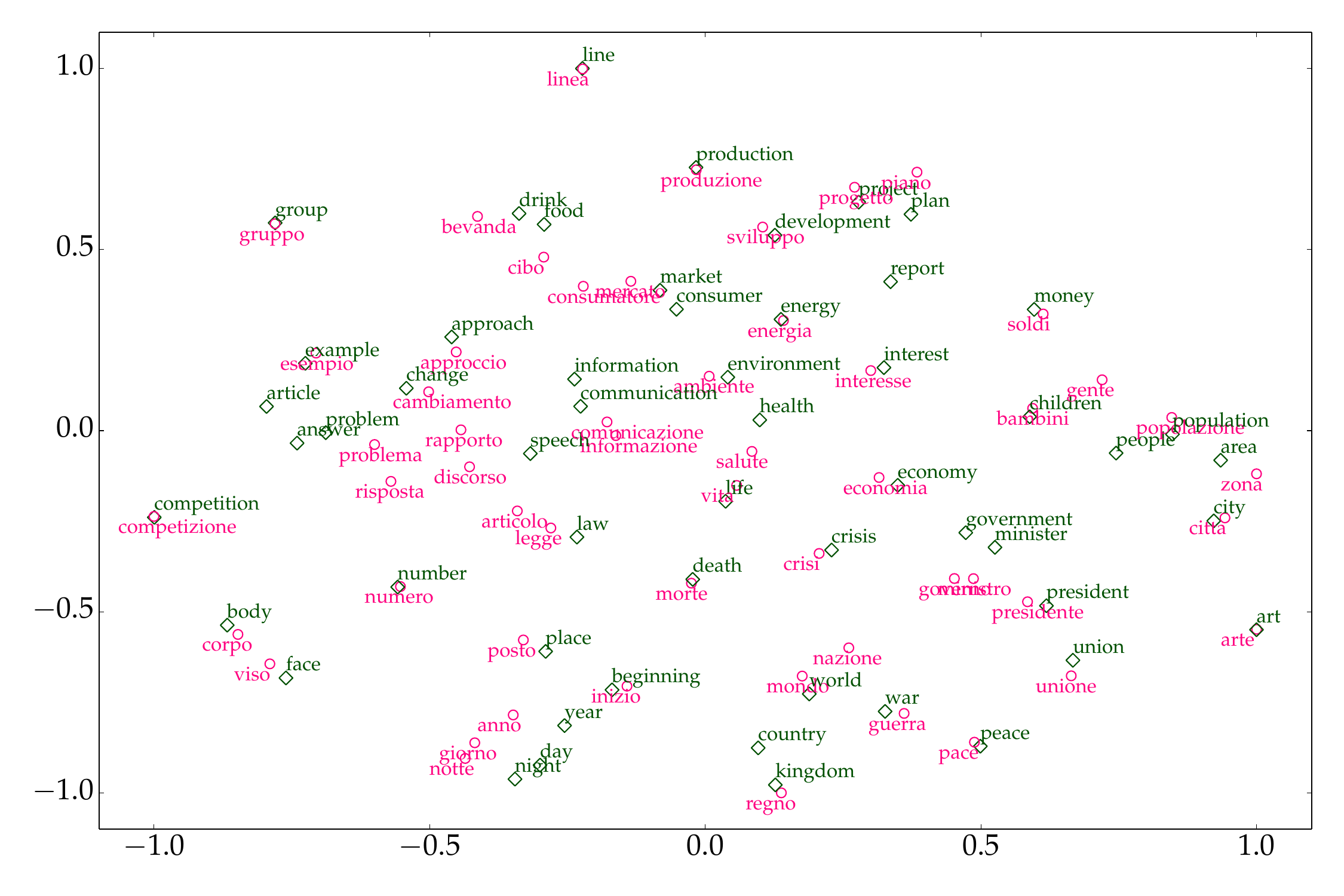}
\end{subfigure}
\caption{Unaligned monolingual word embeddings (left) and word embeddings projected into a joint cross-lingual embedding space (right). Embeddings are visualized with t-SNE.}
\label{fig:shared_embedding_space}
\end{figure}

Many models for learning cross-lingual embeddings have been proposed in recent years. In this survey, we will give a comprehensive overview of existing cross-lingual word embedding models. One of the main goals of this survey is to show the similarities and differences between these approaches. To facilitate this, we first introduce a common notation and terminology in Section \ref{sec:notation_terminology}. Over the course of the survey, we then show that existing cross-lingual word embedding models can be seen as optimizing very similar objectives, where the main source of variation is due to the data used, the monolingual and regularization objectives employed, and how these are optimized. As many cross-lingual word embedding models are inspired by monolingual models, we introduce the most commonly used monolingual embedding models in Section \ref{sec:monolingual_models}. We then motivate and introduce one of the main contributions of this survey, a typology of cross-lingual embedding models in Section \ref{sec:typology}. The typology is based on the main differentiating aspect of cross-lingual embedding models: the nature of the data they require, in particular the type of alignment across languages (alignment of words, sentences, or documents), and whether data is assumed to be parallel or just comparable (about the same topic). The typology allows us to outline similarities and differences more concisely, but also starkly contrasts focal points of research with fruitful directions that have so far gone mostly unexplored. 

Since the idea of cross-lingual representations of words pre-dates word embeddings, we provide a brief history of cross-lingual word representations in Section \ref{sec:brief_history}. Subsequent sections are dedicated to each type of alignment. We discuss cross-lingual word embedding algorithms that rely on word-level alignments in Section \ref{sec:word_level_alignment_models}. Such methods can be further divided into mapping-based approaches, approaches based on pseudo-bilingual corpora, and joint methods. We show that these approaches, \emph{modulo} optimization strategies and hyper-parameters, are nevertheless often equivalent. We then discuss approaches that rely on sentence-level alignments in Section \ref{sec:sentence_level_alignment_models}, and models that require document-level alignments in Section \ref{sec:document_level_alignment_models}. In Section \ref{sec:multilingual_training}, we describe how many bilingual approaches that deal with a pair of languages can be extended to the multilingual setting. We subsequently provide an extensive discussion of the tasks, benchmarks, and challenges of the evaluation of cross-lingual embedding models in Section \ref{sec:evaluation} and outline applications in Section \ref{sec:applications}. We present general challenges and future research directions in learning cross-lingual word representations in Section \ref{sec:challenges}. Finally, we provide our conclusions in Section \ref{sec:conclusion}.

This survey makes the following contributions:
\begin{enumerate}
\item It proposes a general typology that characterizes the differentiating features of cross-lingual word embedding models and provides a compact overview of these models.
\item It standardizes terminology and notation and shows that many cross-lingual word embedding models can be cast as optimizing nearly the same objective functions.
\item It provides a proof that connects the three types of word-level alignment models and shows that these models are optimizing roughly the same objective.
\item It critically examines the standard ways of evaluating cross-lingual embedding models.
\item It describes multilingual extensions for the most common types of cross-lingual embedding models.
\item It outlines outstanding challenges for learning cross-lingual word embeddings and provides suggestions for fruitful and unexplored research directions.
\end{enumerate}

%% file: 2-Notation.tex
\section{Notation and Terminology} \label{sec:notation_terminology}

For clarity, we list all notation used throughout this survey in Table \ref{tab:notation}. We use bold lower case letters ($\mathbf{x}$) to denote vectors, bold upper case letters ($\mathbf{X}$) to denote matrices, and standard weight letters ($x$) for scalars. We use subscripts with bold letters ($\mathbf{x}_i$) to refer to entire rows or columns and subscripts with standard weight letters for specific elements ($x_i$).

Let $\mathbf{X}^\ell \in \mathbb{R}^{|V^\ell| \times d}$ be a word embedding matrix that is learned for the $\ell$-th of $L$ languages where $V^\ell$ is the corresponding vocabulary and $d$ is the dimensionality of the word embeddings. We will furthermore refer to the word embedding of the $i$-th word in language $\ell$ with the shorthand $\mathbf{x}^\ell_i$ or $\mathbf{x}_i$ if language $\ell$ is clear from context. We will refer to the word corresponding to the $i$-th word embedding $\mathbf{x}_i$ as $w_i$ where $w_i$ is a string. To make this correspondence clearer, we will in some settings slightly abuse index notation and use $\mathbf{x}_{w_i}$ to indicate the embedding corresponding to word $w_i$. We will use $i$ to index words based on their order in the vocabulary $V$, while we will use $k$ to index words based on their order in a corpus $\mathcal{C}$.

Some monolingual word embedding models use a separate embedding for words that occur in the context of other words. We will use $\tilde{x}_i$ as the embedding of the $i$-th context word and detail its meaning in the next section. Most approaches only deal with two languages, a source language $s$ and a target language $t$.

\begin{table}[t]
\centering
\begin{tabular}{r l}
\toprule
Symbol & Meaning \\
\midrule
$\mathbf{X}$ & word embedding matrix \\
$V$ & vocabulary \\
$d$ & word embedding dimensionality \\
$\mathbf{x}^\ell_i$ / $\mathbf{x}_i$ / $\mathbf{x}_{w_i}$ & word embedding of the $i$-th word in language $l$ \\
$\tilde{x}_i$ & word embedding of the i-th context word \\
$w_i$ & word pertaining to embedding $\mathbf{x}_i$ \\
$\mathcal{C}$ & corpus of words / aligned sentences used for training \\
$w_k$ & the $k$-th word in a corpus $\mathcal{C}$ \\
$s$ & source language \\
$t$ & target language \\
$\mathbf{W}^{s \rightarrow t}$ / $\mathbf{W}$ & learned transformation matrix between space of $s$ and $t$ \\
$n$ & number of words used as seed words for learning $\mathbf{W}$\\
$\tau$ & function mapping from source words to their translations\\
$\mathbf{C}^s$ & monolingual co-occurrence matrix in language $s$ \\
$C$ & size of context window around a center word\\
$\mathbf{A}^{s\rightarrow t}$ & cross-lingual co-occurrence matrix / alignment matrix \\
$sent^s_i$ & $i$-th sentence in language $s$ \\
$\mathbf{y}^s_i$ & representation of $i$-th sentence in language $s$ \\
$doc^s_i$ & $i$-th document in language $s$ \\
$\mathbf{z}^s_i$ & representation of $i$-th document in language $s$ \\
$\underline{\mathbf{X}^s}$ & $\mathbf{X}^s$ is kept fixed during optimization \\
$\underbrace{\mathcal{L}^1}_\text{1} + \underbrace{\mathcal{L}^2}_\text{2}$ & $\mathcal{L}^1$ is optimized before $\mathcal{L}^2$  \\ 
\bottomrule
\end{tabular}
\caption{Notation used throughout this survey.}
\label{tab:notation}
\end{table}

Some approaches learn a matrix $\mathbf{W}^{s \rightarrow t}$ that can be used to transform the word embedding matrix $\mathbf{X}^s$ of the source language $s$ to that of the target language $t$. We will designate such a matrix by $\mathbf{W}^{s \rightarrow t} \in \mathbb{R}^{d \times d}$ and $\mathbf{W}$ if the language pairing is unambiguous. These approaches often use $n$ source words and their translations as seed words. In addition, we will use $\tau$ as a function that maps from source words $w_i^s$ to their translation $w_i^t$: $\tau: V^s \rightarrow V^t$. Approaches that learn a transformation matrix are usually referred to as \emph{offline} or \emph{mapping} methods. As one of the goals of this survey is to standardize nomenclature, we will use the term \emph{mapping} in the following to designate such approaches.

Some approaches require a monolingual word-word co-occurrence matrix $\mathbf{C}^s$ in language $s$. In such a matrix, every row corresponds to a word $w^s_i$ and every column corresponds to a context word $w^s_j$. $\mathbf{C}_{ij}^s$ then captures the number of times word $w_i$ occurs with context word $w_j$ usually within a window of size $C$ to the left and right of word $w_i$. In a cross-lingual context, we obtain a matrix of alignment counts $\mathbf{A}^{s\rightarrow t} \in \mathbb{R}^{|V^t| \times |V^s|}$, where each element $\mathbf{A}_{ij}^{s\rightarrow t}$ captures the number of times the $i-$th word in language $t$ was aligned with the $j$-th word in language $s$, with each row normalized to sum to $1$.

Finally, as some approaches rely on pairs of aligned sentences, we use $sent_1^s, \ldots, sent_n^s$ to designate sentences in source language $s$ with representations $\mathbf{y}_1^s, \ldots, \mathbf{y}_n^s$ where $\mathbf{y} \in \mathbb{R}^d$. We analogously refer to their aligned sentences in the target language $t$ as $sent_1^t, \ldots, sent_n^t$ with representations $\mathbf{y}_1^t, \ldots, \mathbf{y}_n^t$. We adopt an analogous notation for representations obtained by approaches based on alignments of documents in $s$ and $t$: $doc_1^s, \ldots, doc_n^s$ and $doc_1^t, \ldots, doc_n^t$ with document representations $\mathbf{z}_1^s, \ldots, \mathbf{z}_n^s$ and $\mathbf{z}_1^t, \ldots, \mathbf{z}_n^t$ respectively where $\mathbf{z} \in \mathbb{R}^d$.

Different notations make similar approaches appear different. Using the same notation across our survey facilitates recognizing similarities between the various cross-lingual word embedding models. Specifically, we intend to demonstrate that cross-lingual word embedding models are trained by minimizing roughly the same objective functions, and that differences in objective are unlikely to explain the observed performance differences \cite{Levy2017}.

The class of objective functions minimized by most cross-lingual word embedding methods (if not all), can be formulated as follows: 
\begin{equation}
J = \mathcal{L}^1 + \ldots + \mathcal{L}^\ell + \Omega
\end{equation}
where $\mathcal{L}^\ell$ is the monolingual loss of the $l$-th language and $\Omega$ is a regularization term. A similar loss was also defined by \citeauthor{Upadhyay2016} (2016). As recent work \cite{Levy2014,Levy2015a} shows that many monolingual losses are very similar, one of the main contributions of this survey is to condense the difference between approaches into a regularization term and to detail the assumptions that underlie different regularization terms.

Importantly, how this objective function is optimized is a key characteristic and differentiating factor between different approaches. The joint optimization of multiple non-convex losses is difficult. Most approaches thus take a step-wise approach and optimize one loss at a time while keeping certain variables fixed. Such a step-wise approach is approximate as it does not guarantee to reach even a local optimum.\footnote{Other strategies such as alternating optimization methods, e.g. EM \cite{dempster1977maximum} could be used with the same objective.} In most cases, we will use a longer formulation such as the one below in order to decompose in what order the losses are optimized and which variables they depend on:
\begin{equation}
J = \underbrace{\mathcal{L}(\mathbf{X}^s) + \mathcal{L}(\mathbf{X}^t)}_\text{1} + \underbrace{\Omega(\underline{\mathbf{X}^s}, \underline{\mathbf{X}^t}, \mathbf{W})}_\text{2}
\end{equation}
The underbraces indicate that the two monolingual loss terms on the left, which depend on $\mathbf{X}^s$ and $\mathbf{X}^t$ respectively, are optimized first. Note that this term decomposes into two separate monolingual optimization problems. Subsequently, $\Omega$ is optimized, which depends on $\underline{\mathbf{X}^s}, \underline{\mathbf{X}^t}, \mathbf{W}$. Underlined variables are kept fixed during optimization of the corresponding loss. 

The monolingual losses are optimized by training one of several monolingual embedding models on a monolingual corpus. These models are outlined in the next section.

%% file: 3-Monolingual.tex
\section{Monolingual Embedding Models} \label{sec:monolingual_models}

The majority of cross-lingual embedding models take inspiration from and extend monolingual word embedding models to bilingual settings, or explicitly leverage monolingually trained models. As an important preliminary, we thus briefly review monolingual embedding models that have been used in the cross-lingual embeddings literature.

\paragraph{Latent Semantic Analysis (LSA)} Latent Semantic Analysis \cite{Deerwester1990} has been one of the most widely used methods for learning dense word representations. LSA is typically applied to factorize a sparse word-word co-occurrence matrix $\mathbf{C}$ obtained from a corpus. A common preprocessing method is to replace every entry in $\mathbf{C}$ with its pointwise mutual information (PMI) \cite{Church1990} score:
\begin{equation}
PMI(w_i, w_j) = \log \frac{p(w_i, w_j)}{p(w_i) p(w_j)} = \log \frac{\#(w_i, w_j) \cdot |\mathcal{C}|}{\#(w_i) \cdot \#(w_j)}
\end{equation}
where $\#(\cdot)$ counts the number of (co-)occurrences in the corpus $\mathcal{C}$. As for unobserved word pairs, $PMI(w_i, w_j) = \log 0 = \infty$, such values are often set to $PMI(w_i, w_j) = 0$, which is also known as positive PMI.

The PMI matrix $\mathbf{P}$ where $\mathbf{P}_{i, j} = PMI(w_i, w_j)$ is then factorized using singular value decomposition (SVD), which decomposes $\mathbf{P}$ into the product of three matrices:
\begin{equation} \label{eq:lsa}
\mathbf{P} = \mathbf{U} \mathbf{\Psi} \mathbf{V}^\top 
\end{equation}
where $\mathbf{U}$ and $\mathbf{V}$ are in column orthonormal form and $\mathbf{\Psi}$ is a diagonal matrix of singular values. We subsequently obtain the word embedding matrix $\mathbf{X}$ by reducing the word representations to dimensionality $k$ the following way:
\begin{equation}
\mathbf{X} = \mathbf{U}_k \mathbf{\Psi}_k
\end{equation}
where $\mathbf{\Psi}_k$ is the diagonal matrix containing the top $k$ singular values and $\mathbf{U}_k$ is obtained by selecting the corresponding columns from $\mathbf{U}$.

\paragraph{Max-margin loss (MML)} \citeauthor{Collobert2008} (2008) learn word embeddings by training a model on a corpus $\mathcal{C}$ to output a higher score for a correct word sequence than for an incorrect one. For this purpose, they use a max-margin or hinge loss\footnote{Equations in the literature slightly differ in how they handle corpus boundaries. To make comparing between different monolingual methods easier, we define the sum as starting with the $(C+1)$-th word in the corpus $\mathcal{C}$ (so that the first window includes the first word $w_1$) and ending with the $(|\mathcal{C}|-C)$-th word (so that the final window includes the last word $w_{|\mathcal{C}|}$).}:
\begin{equation}
\begin{split}
\mathcal{L}_{\text{MML}} = \sum_{k=C+1}^{|\mathcal{C}|-C} \sum_{w' \in V} \max(0, 1 & - f([\mathbf{x}_{w_{k-C}}, \ldots, \mathbf{x}_{w_i}, \ldots, \mathbf{x}_{w_{k+C}}]) \\
& + f([\mathbf{x}_{w_{k-C}}, \ldots, \mathbf{x}_{w'}, \ldots, \mathbf{x}_{w_{k+C}}]))
\end{split}
\end{equation}
The outer sum iterates over all words in the corpus $\mathcal{C}$, while the inner sum iterates over all words in the vocabulary. Each word sequence consists of a center word $w_k$ and a window of $C$ words to its left and right. The neural network, which is given by the function $f(\cdot)$, consumes the sequence of word embeddings corresponding to the window of words and outputs a scalar. Using this max-margin loss, it is trained to produce a higher score for a word window occurring in the corpus (the top term) than a word sequence where the center word is replaced by an arbitrary word $w'$ from the vocabulary (the bottom term).

\paragraph{Skip-gram with negative sampling (SGNS)} Skip-gram with negative sampling \cite{Mikolov2013a} is arguably the most popular method to learn monolingual word embeddings due to its training efficiency and robustness \cite{Levy2015a}. SGNS approximates a language model but focuses on learning efficient word representations rather than accurately modeling word probabilities. It induces representations that are good at predicting surrounding context words given a target word $w_k$.
To this end, it minimizes the negative log-likelihood of the training data under the following \textit{skip-gram} objective:
\begin{equation} \label{eq:sgns}
\mathcal{L}_{\text{SGNS}} = - \dfrac{1}{|\mathcal{C}|-C} \sum\limits_{k=C+1}^{|\mathcal{C}|-C} \sum\limits_{-C \leq j \leq C, j \neq 0} \text{log} \: P(w_{k+j} \:|\: w_k)
\end{equation}
$P(w_{k+j} \:|\: w_k)$ is computed using the softmax function:
\begin{align}
P(w_{k+j} \:|\: w_k) = \dfrac{\text{exp}(\mathbf{\tilde{x}}_{w_{k+j}}{}^\top \mathbf{x}_{w_k})}{\sum_{i=1}^{|V|} \text{exp}(\mathbf{\tilde{x}}_{w_i}{}^\top \mathbf{x}_{w_k})}
\end{align}
where $\mathbf{x}_i$ and $\mathbf{\tilde{x}}_i$ are the word and context word embeddings of word $w_i$ respectively. The skip-gram architecture can be seen as a simple neural network: The network takes as input a one-hot representation of a word $\in \mathbb{R}^{|V|}$ and produces a probability distribution over the vocabulary $\in \mathbb{R}^{|V|}$. The embedding matrix $\mathbf{X}$ and the context embedding matrix $\mathbf{\tilde{X}}$ are simply the input-hidden and (transposed) hidden-output weight matrices respectively. The neural network has no nonlinearity, so is equivalent to a matrix product (similar to Equation \ref{eq:lsa}) followed by softmax. 

As the partition function in the denominator of the softmax is expensive to compute, SGNS uses Negative Sampling, which approximates the softmax to make it computationally more efficient. Negative sampling is a simplification of Noise Contrastive Estimation \cite{Gutmann2012}, which was applied to language modeling by \citeauthor{Mnih2012} (2012). Similar to noise contrastive estimation, negative sampling trains the model to distinguish a target word $w_k$ from negative samples drawn from a `noise distribution' $P_n$. In this regard, it is similar to MML as defined above, which ranks true sentences above noisy sentences. Negative sampling is defined as follows:
\begin{equation} \label{eq:negative_sampling}
P(w_{k+j} \:|\: w_k) = \log \sigma(\mathbf{\tilde{x}}_{w_{k+j}}{}^\top \mathbf{x}_{w_k}) + \sum^N_{i=1} \mathbb{E}_{w_i \sim P_n} \log \sigma(-\mathbf{\tilde{x}}_{w_i}{}^\top \mathbf{x}_{w_k})
\end{equation}
where $\sigma$ is the sigmoid function $\sigma(x) = 1/(1+e^{-x})$ and $N$ is the number of negative samples. The distribution $P_n$ is empirically set to the unigram distribution raised to the $3/4^{th}$ power. \citeauthor{Levy2014} \citeyear{Levy2014} observe that negative sampling does not in fact minimize the negative log-likelihood of the training data as in Equation \ref{eq:sgns}, but rather implicitly factorizes a shifted PMI matrix similar to LSA.

\paragraph{Continuous bag-of-words (CBOW)} While skip-gram predicts each context word separately from the center word, 
continuous bag-of-words jointly predicts the center word based on all context words. The model receives as input a window of $C$ context words and seeks to predict the target word $w_k$ by minimizing the CBOW objective:
\begin{equation}
\mathcal{L}_{\text{CBOW}} = - \dfrac{1}{|\mathcal{C}|-C} \sum\limits_{k=C+1}^{|\mathcal{C}|-C} \text{log} \: P(w_k \:|\: w_{k-C}, \ldots, w_{k-1}, w_{k+1}, \ldots, w_{k+C})
\end{equation}

\begin{equation} \label{eq:cbow}
P(w_k \:|\: w_{k-C}, \ldots, w_{k+C}) = \dfrac{\text{exp}(\mathbf{\tilde{x}}_{w_k}{}^\top \mathbf{\bar{x}}_{w_k})}{\sum_{i=1}^{|V|} \text{exp}(\mathbf{\tilde{x}}_{w_i}{}^\top \mathbf{\bar{x}}_{w_k})}
\end{equation}
where $\mathbf{\bar{x}}_{w_k}$ is the sum of the word embeddings of the words $w_{k-C}, \ldots, w_{k+C}$, i.e. $\mathbf{\bar{x}}_{w_k} = \sum_{-C \leq j \leq C, j \neq 0} \mathbf{x}_{w_{k+j}}$.
The CBOW architecture is typically also trained with negative sampling for the same reason as the skip-gram model.

\paragraph{Global vectors (GloVe)} Global vectors \cite{Pennington2014} allows us to learn word representations via matrix factorization. GloVe minimizes the difference between the dot product of the embeddings of a word $\mathbf{x}_{w_i}$ and its context word $\mathbf{\tilde{x}}_{w_j}$ and the logarithm of their number of co-occurrences within a certain window size\footnote{GloVe favors slightly larger window sizes (up to 10 words to the right and to the left of the target word) compared to SGNS \cite{Levy2015a}.}:

\begin{equation}
\mathcal{L}_{\text{GloVe}} = \sum\limits^{|V|}_{i,j=1} f(\mathbf{C}_{ij}) (\mathbf{x}_{w_i}{}^\top \mathbf{\tilde{x}}_{w_j} + b_i + \tilde{b}_j - \text{log} \: \mathbf{C}_{ij})^2
\end{equation}
where $b_i$ and $\tilde{b}_j$ are the biases corresponding to word $w_i$ and word $w_j$, $\mathbf{C}_{ij}$ captures the number of times word $w_i$ occurs with word $w_j$, and $f(\cdot)$ is a weighting function that assigns relatively lower weight to rare and frequent co-occurrences. If we fix $b_i = \log \#(w_i)$ and $\tilde{b}_j = \log \#(w_j)$, then GloVe is equivalent to factorizing a PMI matrix, shifted by $\log |\mathcal{C}|$ \cite{Levy2015a}.

%% file: 4-Typology.tex
\section{Cross-Lingual Word Embedding Models: Typology} \label{sec:typology}

Recent work on cross-lingual embedding models suggests that the actual choice of bilingual supervision signal---that is, the data a method requires to learn to align a cross-lingual representation space---is more important for the final model performance than the actual underlying architecture \cite{Upadhyay2016,Levy2017}.
In other words, large differences between models typically stem from their data requirements, while other fine-grained differences are artifacts of the chosen architecture, hyper-parameters, and additional tricks and fine-tuning employed. This directly mirrors the argument raised by \citeauthor{Levy2015a} (2015) regarding monolingual embedding models: They observe that the architecture is less important as long as the models are trained under identical conditions on the same type (and amount) of data.

We therefore base our typology on the data requirements of the cross-lingual word embedding methods, as this accounts for much of the variation in performance. In particular, methods differ with regard to the data they employ along the following two dimensions:

\begin{enumerate}
	\item \textbf{Type of alignment}: Methods use different types of bilingual supervision signals (at the level of words, sentences, or documents), which introduce stronger or weaker supervision.
    \item \textbf{Comparability}: Methods require either \emph{parallel} data sources, that is, exact translations in different languages or \emph{comparable} data that is only similar in some way.
	
\end{enumerate}

\begin{table}
\centering
\begin{tabular}{l|ll}
\toprule
& Parallel & Comparable \\
\midrule Word & Dictionaries & Images \\
Sentence & Translations & Captions \\
Document & - & Wikipedia \\
\bottomrule
\end{tabular}
\caption{Nature and alignment level of bilingual data sources required by cross-lingual embedding models.}
\label{tab:types_cross-lingual_embedding_models}
\end{table}

\begin{figure}[t!]
    \centering
    \begin{subfigure}[t]{0.19\textwidth}
        \centering
       \includegraphics[height=1.2in]{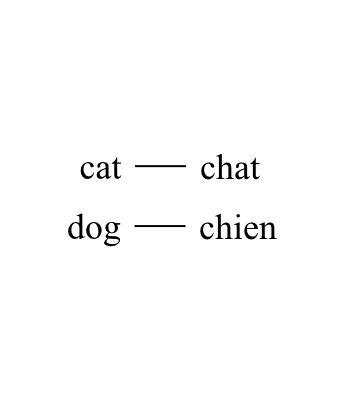}
        \caption{Word, par.} \label{fig:word_par}
    \end{subfigure}%
    ~ 
    \begin{subfigure}[t]{0.19\textwidth}
        \centering
       \includegraphics[height=1.2in]{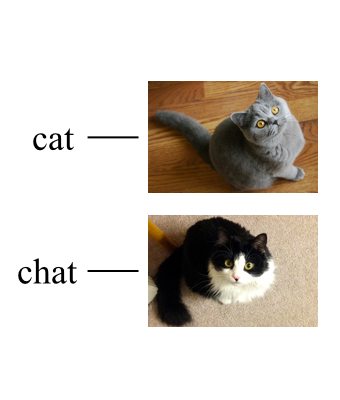}
        \caption{Word, comp.} \label{fig:word_comp}
    \end{subfigure}
    \begin{subfigure}[t]{0.19\textwidth}
        \centering
       \includegraphics[height=1.2in]{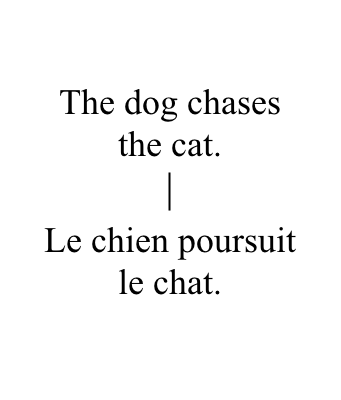}
        \caption{Sentence, par.} \label{fig:sent_par}
    \end{subfigure}
    \begin{subfigure}[t]{0.19\textwidth}
        \centering
       \includegraphics[height=1.2in]{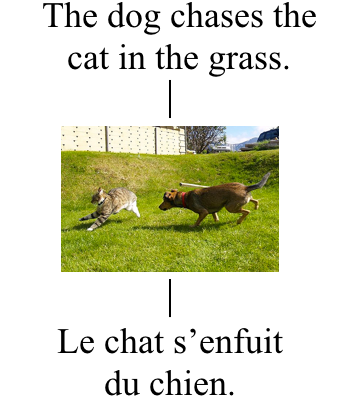}
        \caption{Sentence, comp.} \label{fig:sent_comp}
    \end{subfigure}
    \begin{subfigure}[t]{0.20\textwidth}
        \centering
       \includegraphics[height=1.2in]{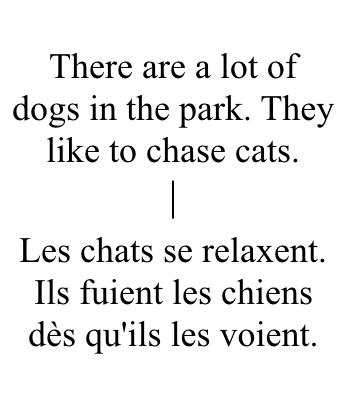}
        \caption{Doc., comp.} \label{fig:doc_comp}
    \end{subfigure}
    \caption{Examples for the nature and type of alignment of data sources. Par.: parallel. Comp.: comparable. Doc.: document. From left to right, word-level parallel alignment in the form of a bilingual lexicon (\ref{fig:word_par}), word-level comparable alignment using images obtained with Google search queries (\ref{fig:word_comp}), sentence-level parallel alignment with translations (\ref{fig:sent_par}), sentence-level comparable alignment using translations of several image captions (\ref{fig:sent_comp}), and document-level comparable alignment using similar documents (\ref{fig:doc_comp}).} \label{fig:examples_data_types}
\end{figure}

In particular, there are three different types of alignments that are possible, which are required by different methods. We discuss the typical data sources for both parallel and comparable data based on the following alignment signals:

\begin{enumerate}
\item \textbf{Word alignment}: Most approaches use parallel word-aligned data in the form of bilingual or cross-lingual dictionary with pairs of translations between words in different languages \cite{Mikolov2013b,Faruqui2014}. Parallel word alignment can also be obtained by automatically aligning words in a parallel corpus (see below), which can be used to produce bilingual dictionaries. Throughout this survey, we thus do not differentiate between the source of word alignment, whether it comes from type-aligned (dictionaries) or token-aligned data (automatically aligned parallel corpora). Comparable word-aligned data, even though more plentiful, has been leveraged less often and typically involves other modalities such as images \cite{Bergsma2011,Kiela2015b}.
\item \textbf{Sentence alignment}: Sentence alignment requires a parallel corpus, as commonly used in machine translation (MT). Methods typically use the Europarl corpus \cite{Koehn2005}, which consists of sentence-aligned text from the proceedings of the European parliament, and is perhaps the most common source of training data for MT models \cite{Hermann2013,Lauly2013}. Other methods use available word-level alignment information \cite{Zou2013,Shi2015}.
There has been some work on extracting parallel data from comparable corpora \cite{Munteanu2006}, but no-one has so far trained cross-lingual word embeddings on such data. 
Comparable data with sentence alignment may again leverage another modality, such as captions of the same image or similar images in different languages, which are not translations of each other \cite{Calixto2017,Gella2017}.
\item \textbf{Document alignment}: Parallel document-aligned data requires documents in different languages that are translations of each other. This is rare, as parallel documents typically consist of aligned sentences \cite{Hermann2014}. Comparable document-aligned data is more common and can occur in the form of documents that are topic-aligned (e.g. Wikipedia) or class-aligned (e.g. sentiment analysis and multi-class classification datasets) \cite{Vulic:2013emnlp,Mogadala2016}.
\end{enumerate}

We summarize the different types of data required by cross-lingual embedding models along these two dimensions in Table \ref{tab:types_cross-lingual_embedding_models} and provide examples for each in Figure \ref{fig:examples_data_types}. Over the course of this survey we will show that models that use a particular type of data are mostly variations of the same or similar architectures. We present our complete typology of cross-lingual embedding models in Table \ref{tab:cross-lingual_embedding_models}, aiming to provide an exhaustive overview by classifying each model (we are aware of) into one of the corresponding model types. We also provide a more detailed overview of the monolingual objectives and regularization terms used by every approach towards the end of this survey in Table \ref{tab:overview_objectives}.

\begin{table}
\centering
\resizebox*{\textwidth}{\textheight}{%
{\footnotesize
\begin{tabular}{l|ll}
\toprule
& {\bf Parallel} & {\bf Comparable}  \\
\midrule
\multirow{16}{*}{\shortstack[l]{\bf Word\\\bf---Mapping}} & Mikolov et al. (2013b) & Bergsma and Van Durme (2011) \\
& Faruqui and Dyer (2014) & Kiela et al. (2015) \\
& Lazaridou et al. (2015) & Vulić et al. (2016) \\
& Dinu et al. (2015) &  \\
& Xing et al. (2015) & \\ 
& Lu et al. (2015) & \\
& Vulić and Korhonen (2016) &  \\
& Ammar et al. (2016b) & \\
& Zhang et al. (2016b, 2017ab) &  \\
& Artexte et al. (2016, 2017, 2018ab) & \\
& Smith et al. (2017) & \\
& Hauer et al. (2017) & \\
& Mrkšić et al. (2017b) & \\
& Conneau et al. (2018a) & \\
& Joulin et al. (2018) & \\
& Alvarez-Melis and Jaakkola (2018) & \\
& Ruder et al. (2018) & \\
& Glavaš et al. (2019) & \\
\midrule
\multirow{4}{*}{\shortstack[l]{\bf Word\\\bf---Pseudo-\\\bf bilingual}} & Xiao and Guo (2014) & Duong et al. (2015) \\
& \multicolumn{2}{c}{Gouws and Søgaard (2015)} \\
& Duong et al. (2016) & \\
& Adams et al. (2017) & \\
\midrule
\multirow{2}{*}{\shortstack[l]{\bf Word\\\bf---Joint}} & Klementiev et al. (2012) \\
& Kočiský et al. (2014) & \\\midrule \midrule
\multirow{5}{*}{\shortstack[l]{\bf Sentence\\\bf---Matrix\\\bf factorization}} & Zou et al. (2013) & \\
& Shi et al. (2015) & \\
& Gardner et al. (2015) & \\
& Guo et al. (2015) & \\
& Vyas and Carpuat (2016) &  \\ \midrule
\multirow{2}{*}{\shortstack[l]{\bf Sentence\\\bf---Compositional}} & Hermann and Blunsom (2013, 2014) & \\
& Soyer et al. (2015) & \\ \midrule
\multirow{2}{*}{\shortstack[l]{\bf Sentence\\\bf---Autoencoder}} & Lauly et al. (2013) & \\
& Chandar et al. (2014) & \\ \midrule
\multirow{4}{*}{\shortstack[l]{\bf Sentence\\\bf---Skip-gram}} & Gouws et al. (2015) & \\
& Luong et al. (2015) & \\
& Coulmance et al. (2015) & \\
& Pham et al. (2015) & \\ \midrule
\multirow{2}{*}{\shortstack[l]{\bf Sentence\\\bf---Other}} & Levy et al. (2017) & Calixto et al. (2017) \\
& Rajendran et al. (2016) & Gella et al. (2017) \\\midrule \midrule
\multirow{3}{*}{\bf Document} & & Vulić and Moens (2013a, 2014, 2016) \\
& & Søgaard et al. (2015)\\
& & Mogadala and Rettinger (2016) \\\bottomrule
\end{tabular}
}%
}
\caption{Cross-lingual embedding models ordered by data requirements.}
\label{tab:cross-lingual_embedding_models}
\end{table}

%% file: 5-History.tex
\section{A Brief History of Cross-Lingual Word Representations} \label{sec:brief_history}
We provide a brief overview of the historical lineage of cross-lingual word embedding models. In brief, while cross-lingual word embeddings is a novel phenomenon, many of the high-level ideas that motivate current research in this area, can be found 
in work that pre-dates the popular introduction of word embeddings. This includes work on learning cross-lingual word representations from seed lexica, parallel data, or document-aligned data, as well as ideas on learning from limited bilingual supervision.

Language-independent representations have been proposed for decades, many of which rely on abstract linguistic labels instead of lexical features \cite{Aone:1993acl,Schultz:2001sc}. This is also the strategy used in early work on so-called {\em delexicalized}~cross-lingual and domain transfer \cite{Zeman:2008ijcnlp,Soegaard:11:dp,McDonald:2011emnlp,Cohen:ea:11,Tackstrom:2012naacl,Henderson:2014slt}, as well as in work on inducing cross-lingual word clusters \cite{Tackstrom:2012naacl,Faruqui:2013acl}, and cross-lingual word embeddings relying on syntactic/POS contexts \cite{Duong:2015conll,Dehouck:2017eacl}.\footnote{Along the same line, the recent initiative on providing cross-linguistically consistent sets of such labels \cite<e.g., Universal Dependencies,>{Nivre:2015ud} facilitates cross-lingual transfer and offers further support to the induction of word embeddings across languages \cite{Vulic:2017eacl,Vulic:2017conll}.} The ability to represent lexical items from two different languages in a shared cross-lingual space supplements seminal work in cross-lingual transfer by providing fine-grained word-level links between languages; see work in cross-lingual dependency parsing \cite{Ammar:2016tacl,Zeman:2017conll} and natural language understanding systems (Mrkšić et al., 2017b).

Similar to our typology of cross-lingual word embedding models outlined in Table~\ref{tab:cross-lingual_embedding_models} based on bilingual data requirements from Table~\ref{tab:types_cross-lingual_embedding_models}, one can  also arrange older cross-lingual representation architectures into similar categories. A traditional approach to cross-lingual vector space induction was based on high-dimensional context-counting vectors where each dimension encodes the (weighted) co-occurrences with a specific context word in each of the languages. The vectors are then \textit{mapped} into a single cross-lingual space using a seed bilingual dictionary containing paired context words from both sides \cite[inter alia]{Rapp:1999acl,Gaussier:2004acl,Laroche:2010coling,Tamura:2012emnlp}. This approach is an important predecessor to the cross-lingual word embedding models described in Section~\ref{sec:word_level_alignment_models}. Similarly, the bootstrapping technique developed for traditional context-counting approaches \cite{Peirsman:2010naacl,Vulic:2013emnlp} is an important predecessor to recent iterative self-learning techniques used to limit the bilingual dictionary seed supervision needed in mapping-based approaches \cite{Hauer2017,Artetxe:2017acl}. The idea of CCA-based word embedding learning \cite<see later in Section~\ref{sec:word_level_alignment_models};>{Faruqui2014,Lu:2015naacl} was also introduced a decade earlier \cite{Haghighi:2008acl}; their word~additionally discussed the idea of combining orthographic subword features with distributional signatures for cross-lingual representation learning: This idea re-entered the literature recently \cite{Heyman:2017eacl}, only now with much better performance.

Cross-lingual word embeddings can also be directly linked to the work on word alignment for statistical machine translation \cite{Brown:1993cl,Och:2003cl}. \citeauthor{Levy2017} (2017) stress that word translation probabilities extracted from sentence-aligned parallel data by IBM alignment models can also act as the cross-lingual semantic similarity function in lieu of the cosine similarity between word embeddings. Such word translation tables are then used to induce bilingual lexicons. For instance, aligning each word in
a given source language sentence with the most similar target language word from the target language sentence is exactly the same greedy decoding algorithm that is implemented in IBM Model 1. Bilingual dictionaries and cross-lingual word clusters derived from word alignment links can be used to boost cross-lingual transfer for applications such as syntactic parsing \cite{Tackstrom:2012naacl,Durrett:2012emnlp}, POS tagging \cite{Agic:2015acl}, or semantic role labeling \cite{Kozhevnikov:2013acl} by relying on shared lexical information stored in the bilingual lexicon entries. Exactly the same functionality can be achieved by cross-lingual word embeddings. However, cross-lingual word embeddings have another advantage in the era of neural networks: the continuous representations can be plugged into current end-to-end neural architectures directly as sets of lexical features. 

A large body of work on multilingual probabilistic topic modeling \cite{Vulic:2015ipm,Boyd:2017book} also extracts shared cross-lingual word spaces, now by means of conditional latent topic probability distributions: two words with similar distributions over the induced latent variables/topics are considered semantically similar. The learning process is again steered by the data requirements. The early days witnessed the use of pseudo-bilingual corpora constructed by merging aligned document pairs, and then applying a monolingual representation model such as LSA \cite{Landauer:1997} or LDA \cite{Blei:2003jmlr} on top of the merged data \cite{Littman:1998,DeSmet:2011pakdd}. This approach is very similar to the pseudo-cross-lingual approaches discussed in Section~\ref{sec:word_level_alignment_models} and Section~\ref{sec:document_level_alignment_models}. More recent topic models learn on the basis of parallel word-level information, enforcing word pairs from seed bilingual lexicons (again!) to obtain similar topic distributions \cite{Boyd:2009uai,Zhang:2010acl,Boyd:2010emnlp,Jagarlamudi:2010ecir}. In consequence, this also influences topic distributions of related words not occurring in the dictionary. Another group of models utilizes alignments at the document level \cite{Mimno:2009emnlp,Platt:2010emnlp,Vulic:2011acl,Fukumasu:2012nips,Heyman:2016dami} to induce shared topical spaces. The very same level of supervision (i.e., document alignments) is used by several cross-lingual word embedding models, surveyed in Section~\ref{sec:document_level_alignment_models}. Another embedding model based on the document-aligned Wikipedia structure \cite{Sogaard2015} bears resemblance with the cross-lingual Explicit Semantic Analysis model \cite{Gabrilovich:2006aaai,Hassan:2009emnlp,Sorg:2012dke}.

All these ``historical'' architectures measure the strength of cross-lingual word similarities through metrics defined in the cross-lingual space: e.g., Kullback-Leibler or Jensen-Shannon divergence (in a topic space), or vector inner products (in sparse context-counting vector space),
and are therefore applicable to NLP tasks that rely cross-lingual similarity scores. 
The pre-embedding architectures and more recent cross-lingual word embedding methods have been applied to an overlapping set of evaluation tasks and applications, ranging from bilingual lexicon induction to cross-lingual knowledge transfer, including cross-lingual parser transfer \cite{Tackstrom:2012naacl,Ammar:2016tacl}, cross-lingual document classification \cite{Gabrilovich:2006aaai,DeSmet:2011pakdd,Klementiev2012,Hermann2014}, cross-lingual relation extraction \cite{Faruqui:2015naaclshort}, etc. In summary, while sharing the goals and assumptions of older cross-lingual architectures, cross-lingual word embedding methods have capitalized on the recent methodological and algorithmic advances in the field of representation learning, owing their wide use to their simplicity, efficiency and handling of large corpora, as well as their relatively robust performance across domains.




%% file: 6-1-Word-alignment-mapping.tex
\section{Word-Level Alignment Models} \label{sec:word_level_alignment_models}

In the following, we will now discuss different types of the current generation of cross-lingual embedding models, starting with models based on word-level alignment. Among these, models based on parallel data are more common. 

\subsection{Word-level Alignment Methods with Parallel Data}

We distinguish three methods that use parallel word-aligned data: 

\begin{itemize}
\item[a)] \textbf{Mapping-based approaches} that first train monolingual word representations independently on large monolingual corpora and then seek to learn a transformation matrix that maps representations in one language to the representations of the other language. They learn this transformation from word alignments or bilingual dictionaries (we do not see a need to distinguish between the two). 
\item[b)] \textbf{Pseudo-multi-lingual corpora-based approaches} that use monolingual word embedding methods on automatically constructed (or corrupted) corpora that contain words from both the source and the target language. 
\item[c)] \textbf{Joint methods} that take parallel text as input and minimize the source and target language monolingual losses jointly with the cross-lingual regularization term. 
\end{itemize}

We will show that {\em modulo} optimization strategies, these approaches are equivalent. Before discussing the first category of methods, we briefly introduce two concepts that are of relevance in these and the subsequent sections.

\paragraph{Bilingual lexicon induction} Bilingual lexicon induction is the intrinsic task that is most commonly used to evaluate current cross-lingual word embedding models. Briefly, given a list of $N$ language word forms $w_1^s, \ldots, w_N^s$, the goal is to determine the most appropriate translation $w_i^t$, for each query form $w_i^s$. This is commonly accomplished by finding a target language word whose embedding $\mathbf{x}^t_i$ is the nearest neighbour to the source word embedding $\mathbf{x}_i^s$ in the shared semantic space, where similarity is usually computed as the cosine similarity between their embeddings. See Section \ref{sec:evaluation} for more details.

\paragraph{Hubness} Hubness \cite{radovanovic2010hubs} is a phenomenon observed in high-dimensional spaces where some points (known as \emph{hubs}) are the nearest neighbours of many other points. As translations are assumed to be nearest neighbours in cross-lingual embedding space, hubness has been reported to affect cross-lingual word embedding models.

\subsubsection{Mapping-based Approaches}

Mapping-based approaches are by far the most prominent category of cross-lingual word embedding models and---due to their conceptual simplicity and ease of use---are currently the most popular. Mapping-based approaches aim to learn a mapping from the monolingual embedding spaces to a joint cross-lingual space. Approaches in this category differ along multiple dimensions:
\begin{enumerate}
\item \textbf{The mapping method} that is used to transform the monolingual embedding spaces into a cross-lingual embedding space.
\item \textbf{The seed lexicon} that is used to learn the mapping.
\item \textbf{The refinement} of the learned mapping.
\item \textbf{The retrieval} of the nearest neighbours.
\end{enumerate}

\subsubsection*{Mapping Methods}

There are four types of mapping methods that have been proposed:
\begin{enumerate}
\item \textbf{Regression methods} map the embeddings of the source language to the target language space by maximizing their similarity. 
\item \textbf{Orthogonal methods} map the embeddings in the source language to maximize their similarity with the target language embeddings, but constrain the transformation to be orthogonal.
\item \textbf{Canonical methods} map the embeddings of both languages to a new shared space, which maximizes their similarity.
\item \textbf{Margin methods} map the embeddings of the source language to maximize the margin between correct translations and other candidates.
\end{enumerate}

\begin{figure}[!hb]
      \centering
         \includegraphics[width=0.6 \linewidth]{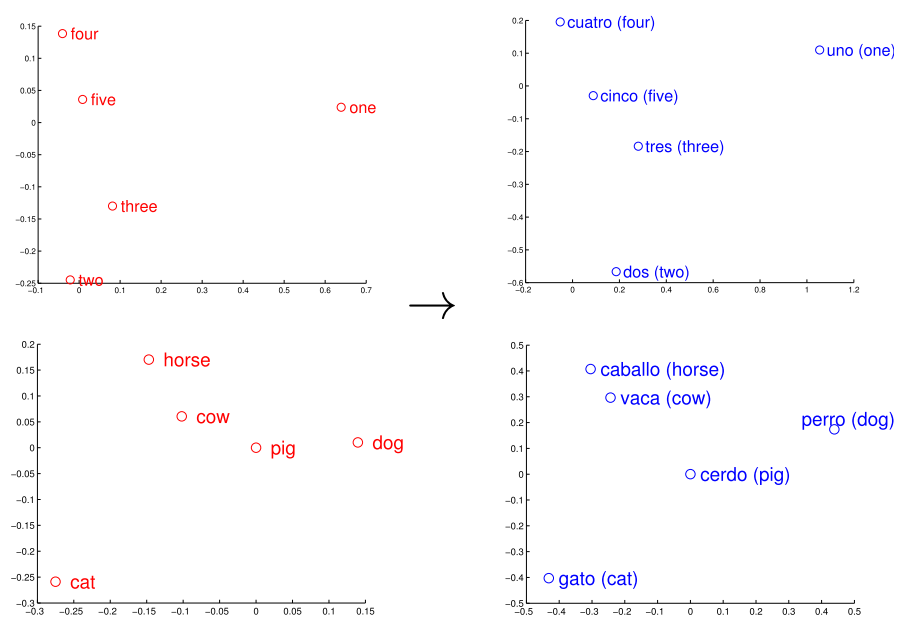}
	\caption{Similar geometric relations between numbers and animals in English and Spanish \cite{Mikolov2013b}. Words embeddings are projected to two dimensions using PCA and were manually rotated to emphasize similarities.}
	\label{fig:geometric_relations}
\end{figure}

\paragraph{Regression methods}

One of the most influential methods for learning a mapping is the linear transformation method by \citeauthor{Mikolov2013b} \citeyear{Mikolov2013b}. The method is motivated by the observation that words and their translations show similar geometric constellations in monolingual embedding spaces after an appropriate linear transformation is applied, as illustrated in Figure \ref{fig:geometric_relations}. This suggests that it is possible to transform the vector space of a source language $s$ to the vector space of the target language $t$ by learning a linear projection with a transformation matrix $\mathbf{W}^{s \rightarrow t}$. We use $\mathbf{W}$ in the following if the direction is unambiguous.

Using the $n$ most frequent words from the source language $w^s_1, \ldots, w^s_n$ and their translations $w^t_1, \ldots, w^t_n$ as seed words, they learn $\mathbf{W}$ using stochastic gradient descent by minimising the squared Euclidean distance (mean squared error, MSE) between the previously learned monolingual representations of the source seed word $\mathbf{x}_i^s$ that is transformed using $\mathbf{W}$ and its translation $\mathbf{x}_i^t$ in the bilingual dictionary:

\begin{equation} \label{eq:mikolov-mapping}
\Omega_{\text{MSE}} = \sum\limits^n_{i=1} \|\mathbf{W} \mathbf{x}_i^s - \mathbf{x}_i^t \|^2 
\end{equation}

This can also be written in matrix form as minimizing the squared Frobenius norm of the residual matrix:
\begin{equation}
\Omega_{\text{MSE}} = \|\mathbf{W}\mathbf{X}^s - \mathbf{X}^t\|^2_F 
\end{equation}
where $\mathbf{X}^s$ and $\mathbf{X}^t$ are the embedding matrices of the seed words in the source and target language respectively. Analogously, the problem can be seen as finding a least squares solution to a system of linear equations with multiple right-hand sides:
\begin{equation}
\mathbf{W}\mathbf{X}^s = \mathbf{X}^t
\end{equation}
A common solution to this problem enables calculating $\mathbf{W}$ analytically as $\mathbf{W} = \mathbf{X}^+ \mathbf{X}^t$ where $\mathbf{X}^{+} = (\mathbf{X}^s{}^\top \mathbf{X}^s)^{-1} \mathbf{X}^s{}^\top$ is the Moore-Penrose pseudoinverse.

In our notation, the MSE mapping approach can be seen as optimizing the following objective:
\begin{equation} \label{eq:mse}
J = \underbrace{\mathcal{L}_{\text{SGNS}}(\mathbf{X}^s) + \mathcal{L}_{\text{SGNS}}(\mathbf{X}^t)}_\text{1} + \underbrace{\Omega_{\text{MSE}}(\underline{\mathbf{X}^s}, \underline{\mathbf{X}^t}, \mathbf{W})}_\text{2}
\end{equation}
First, each of the monolingual losses is optimized independently. Second, the regularization term $\Omega_{\text{MSE}}$ is optimized while keeping the induced monolingual embeddings fixed. The basic approach of \citeauthor{Mikolov2013b} \citeyear{Mikolov2013b} has later been adopted by many others who for instance incorporated $\ell_2$ regularization \cite{Dinu2015}. A common preprocessing step that is applied to the monolingual embeddings is to normalize the monolingual embeddings to be unit length. \citeauthor{Xing2015} (2015) argue that this normalization resolves an inconsistency between the metric used for training (dot product) and the metric used for evaluation (cosine similarity).\footnote{For unit vectors, dot product and cosine similarity are equivalent.} \citeauthor{Artetxe2016} (2016) motivate length normalization to ensure that all training instances contribute equally to the objective.

\paragraph{Orthogonal methods} The most common way in which the basic regression method of the previous section has been improved is to constrain the transformation $\mathbf{W}$ to be orthogonal, i.e. $\mathbf{W}^\top \mathbf{W} = \mathbf{I}$. The exact solution under this constraint is $\mathbf{W} = \mathbf{V} \mathbf{U}^\top$ and can be efficiently computed in linear time with respect to the vocabulary size using SVD where $\mathbf{X}^t{}^\top \mathbf{X}^s = \mathbf{U} \mathbf{\Sigma} \mathbf{V}^\top$. This constraint is motivated by \citeauthor{Xing2015} (2015) to preserve length normalization. \citeauthor{Artetxe2016} (2016) motivate orthogonality as a means to ensure monolingual invariance. An orthogonality constraint has also been used to regularize the mapping \cite{Zhang2016,Zhang2017} and has been motivated theoretically to be self-consistent \cite{Smith2017}.

\paragraph{Canonical methods} Canonical methods map the embeddings in both languages to a shared space using Canonical Correlation Analysis (CCA). \citeauthor{Haghighi:2008acl} \citeyear{Haghighi:2008acl} were the first to use this method for learning translation lexicons for the words of different languages. \citeauthor{Faruqui2014} (2014) later applied CCA to project words from two languages into a shared embedding space. Whereas linear projection only learns one transformation matrix $\mathbf{W}^{s\rightarrow t}$ to project the embedding space of a source language into the space of a target language, CCA learns a transformation matrix for the source and target language $\mathbf{W}^{s\rightarrow}$ and $\mathbf{W}^{t\rightarrow}$ respectively to project them into a new joint space that is different from both the space of $s$ and of $t$.
We can write the correlation between a projected source language embedding vector $\mathbf{W}^{s\rightarrow} \mathbf{x}_i^s$ and its corresponding projected target language embedding vector $\mathbf{W}^{t\rightarrow} \mathbf{x}_i^t$ as:
\begin{equation}
\rho(\mathbf{W}^{s\rightarrow} \mathbf{x}_i^s, \mathbf{W}^{t\rightarrow} \mathbf{x}_i^t) = \frac{\text{cov}(\mathbf{W}^{s\rightarrow} \mathbf{x}_i^s, \mathbf{W}^{t\rightarrow} \mathbf{x}_i^t)}{\sqrt{\text{var}(\mathbf{W}^{s\rightarrow} \mathbf{x}_i^s) \text{var}(\mathbf{W}^{t \rightarrow} \mathbf{x}_i^t)}}
\end{equation}
where $\text{cov}(\cdot, \cdot)$ is the covariance and $\text{var}(\cdot)$ is the variance. CCA then aims to maximize the correlation (or analogously minimize the negative correlation) between the projected vectors $\mathbf{W}^{s\rightarrow} \mathbf{x}_i^s$ and $\mathbf{W}^{t\rightarrow} \mathbf{x}_i^t$:
\begin{equation}
\Omega_{\text{CCA}} = - \sum^n_{i=1} \rho(\mathbf{W}^{s\rightarrow} \mathbf{x}_i^s, \mathbf{W}^{t\rightarrow} \mathbf{x}_i^t)
\end{equation}
We can write their objective in our notation as the following:
\begin{equation}
J = \underbrace{\mathcal{L}_{\text{LSA}}(\mathbf{X}^s) + \mathcal{L}_{\text{LSA}}(\mathbf{X}^t)}_\text{1} + \underbrace{\Omega_{\text{CCA}}(\underline{\mathbf{X}^s}, \underline{\mathbf{X}^t}, \mathbf{W}^{s\rightarrow}, \mathbf{W}^{t\rightarrow})}_\text{2}
\end{equation}
\citeauthor{Faruqui2014} propose to use the $80$\% projection vectors with the highest correlation.
\citeauthor{Lu:2015naacl} (2015) incorporate a non-linearity into the canonical method by training two deep neural networks to maximize the correlation between the projections of both monolingual embedding spaces. \citeauthor{Ammar2016} \citeyear{Ammar2016} extend the canonical approach to multiple languages.

\citeauthor{Artetxe2016} \citeyear{Artetxe2016} show that the canonical method is similar to the orthogonal method with dimension-wise mean centering. \citeauthor{Artetxe2018} \citeyear{Artetxe2018} show that regression methods, canonical methods, and orthogonal methods can be seen as instances of a framework that includes optional weightening and de-whitening steps, which further demonstrates the similarity of existing approaches.

\paragraph{Margin methods} \citeauthor{Lazaridou2015} (2015) optimize a max-margin based ranking loss instead of MSE to reduce hubness. This max-margin based ranking loss is essentially the same as the MML \cite{Collobert2008} used for learning monolingual embeddings. Instead of assigning higher scores to correct sentence windows, we now try to assign a higher cosine similarity to word pairs that are translations of each other ($\mathbf{x}_i^s, \mathbf{x}_i^t$; first term below) than random word pairs ($\mathbf{x}_i^s, \mathbf{x}_j^t$; second term): 
\begin{equation}
\Omega_{\text{MML}} = \sum^n_{i=1} \sum^k_{j \neq i} \max\{0, \gamma - \cos(\mathbf{W}\mathbf{x}_i^s, \mathbf{x}_i^t) + \cos(\mathbf{W}\mathbf{x}_i^s, \mathbf{x}_j^t) \}
\end{equation}
The choice of the $k$ negative examples, which we compare against the translations is crucial. \citeauthor{Dinu2015} (2015) propose to select negative examples that are more informative by being near the current projected vector $\mathbf{W}\mathbf{x}_i^s$ but far from the actual translation vector $\mathbf{x}_i^t$. Unlike random intruders, such intelligently chosen intruders help the model identify training instances where the model considerably fails to approximate the target function. In the formulation adopted in this article, their objective becomes:
\begin{equation}
J = \underbrace{\mathcal{L}_{\text{CBOW}}(\mathbf{X}^s) + \mathcal{L}_{\text{CBOW}}(\mathbf{X}^t)}_\text{1} + \underbrace{\Omega_{\text{MML-I}}(\underline{\mathbf{X}^s}, \underline{\mathbf{X}^t}, \mathbf{W})}_\text{2}
\end{equation}
where $\Omega_{\text{MML-I}}$ designates $\Omega_{\text{MML}}$ with intruders as negative examples. More recently, \citeauthor{Joulin2018a} (2018) proposed a margin-based method, which replaces cosine similarity with CSLS, a distance function more suited to bilingual lexicon induction that will be discussed in the retrieval section.

Among the presented mapping approaches, orthogonal methods are the most commonly adopted as the orthogonality constraint improves over the basic regression method.

\subsubsection*{The Seed Lexicon}

The seed lexicon is another core component of any mapping-based approach. In the past, three types of seed lexicons have been used to learn a joint cross-lingual word embedding space:
\begin{enumerate}
    \item An \textbf{off-the-shelf} bilingal lexicon. 
    \item A \textbf{weakly supervised} bilingual lexicon.
    \item A \textbf{learned} bilingual lexicon.
\end{enumerate}

\paragraph{Off-the-shelf} Most early approaches \cite{Mikolov2013b} employed off-the-shelf or automatically generated bilingual lexicons of frequent words. While \citeauthor{Mikolov2013b} \citeyear{Mikolov2013b} used as much as 5000 pairs, later approaches reduce the number of seed pairs, demonstrating that it is feasible to learn a cross-lingual word embedding space with as little as 25 seed pairs \cite{Artetxe:2017acl}. 

\paragraph{Weak supervision} Other approaches employ weak supervision to create seed lexicons based on cognates \cite{Smith2017}, shared numerals \cite{Artetxe:2017acl}, or identically spelled strings \cite{Soegaard2018}. Such weak supervision is easy to obtain and has been shown to produce results that are competitive with off-the-shelf lexicons.

\paragraph{Learned} Recently, approaches have been proposed that learn an initial seed lexicon in a completely unsupervised way. Interestingly, so far, all unsupervised cross-lingual word embedding methods are based on the mapping approach. \citeauthor{Conneau2018} \citeyear{Conneau2018} learn an initial mapping in an adversarial way by additionally training a discriminator to differentiate between projected and actual target language embeddings. \citeauthor{Artetxe2018b} \citeyear{Artetxe2018b} propose to use an initialisation method based on the heuristic that translations have similar similarity distributions across languages. \citeauthor{Hoshen2018} (2018) first project vectors of the $N$ most frequent words to a lower-dimensional space with PCA. They then aim to find an optimal transformation that minimizes the sum of Euclidean distances by learning $\mathbf{W}^{s\rightarrow t}$ and $\mathbf{W}^{t\rightarrow s}$ and enforce cyclical consistency constraints that force vectors round-projected to the other language space and back to remain unchanged. \citeauthor{Alvarez-Melis2018} (2018) solve an optimal transport in order to learn an alignment between the monolingual word embedding spaces.

\subsubsection*{The Refinement}

Many mapping-based approaches propose to refine the mapping to improve the quality of the initial seed lexicon. \citeauthor{VulicKorhonen2016a} (2016) propose to learn a first shared bilingual embedding space based on an existing cross-lingual embedding model. They retrieve the translations of frequent source words in this cross-lingual embedding space, which they use as seed words to learn a second mapping. To ensure that the retrieved translations are reliable, they propose a symmetry constraint: Translation pairs are only retained if their projected embeddings are mutual nearest neighbours in the cross-lingual embedding space. This constraint is meant to reduce hubness and has been adopted later in many subsequent methods that rely heavily on refinement \cite{Conneau2018,Artetxe2018b}.

Rather than just performing one step of refinement, \citeauthor{Artetxe:2017acl} (2017) propose a method that iteratively learns a new mapping by using translation pairs from the previous mapping. Training is terminated when the improvement on the average dot product for the induced dictionary falls below a given threshold from one iteration to the next. \citeauthor{Ruder2018a} (2018) solve a sparse linear assignment problem in order to refine the mapping. As discussed in Section~\ref{sec:brief_history}, the refinement idea is conceptually similar to the work of \citeauthor{Peirsman2011} (2010, 2011) and \citeauthor{Vulic:2013emnlp} \citeyear{Vulic:2013emnlp}, with the difference that earlier approaches were developed within the traditional cross-lingual distributional framework (mapping vectors into the count-based space using a seed lexicon). \citeauthor{Glavas2019} (2019) propose to learn a matrix $\mathbf{W}^{s\rightarrow t}$ and a matrix $\mathbf{W}^{t\rightarrow s}$. They then use the intersection of the translation pairs obtained from both mappings in the subsequent iteration. In practice, one step of refinement is often sufficient as in the second iteration, a large number of noisier word translations are automatically generated \cite{Glavas2019}.

While refinement is less crucial when a large seed lexicon is available, approaches that learn a mapping from a small seed lexicon or in a completely unsupervised way rely on refinement \cite{Conneau2018,Artetxe2018b}.

\subsubsection*{The Retrieval}

Most existing methods retrieve translations as the nearest neighbours of the source word embeddings in the cross-lingual embedding space based on cosine similarity. \citeauthor{Dinu2015} (2015) propose to use a globally corrected neighbour retrieval method instead to reduce hubness. \citeauthor{Smith2017} (2017) propose a similar solution to the hubness issue: they invert the softmax used for finding the translation of a word at test time and normalize the probability over source words rather than target words. \citeauthor{Conneau2018} \citeyear{Conneau2018} propose an alternative similarity measure called cross-domain similarity local scaling (CSLS), which is defined as:
\begin{equation}
\text{CSLS}(\mathbf{W}\mathbf{x}^s, \mathbf{x}^t) = 2 \text{cos}(\mathbf{W} \mathbf{x}^s, \mathbf{x}^t) - r^t(\mathbf{W}\mathbf{x}^s) - r^s(\mathbf{x}^t)
\end{equation}
where $r^t$ is the mean similarity of a target word to its neighbourhood, defined as $r^t(\mathbf{W}\mathbf{x}_s) = \frac{1}{K} \sum_{\mathbf{x}_t \in \mathcal{N}^t(\mathbf{W}\mathbf{x}^s}) \text{cos}(\mathbf{W}\mathbf{x}^s, \mathbf{x}^t)$ where $\mathcal{N}^t(\mathbf{W}\mathbf{x}^s)$ is the neighbourhood of the projected source word. Intuitively, CSLS increases the similarity of isolated word vectors and decreases the similarity of hubs. CSLS has been shown to significantly increase the accuracy of bilingual lexicon induction and is nowadays mostly used in lieu of cosine similarity for nearest neighbour retrieval. \citeauthor{Joulin2018a} (2018) propose to optimize this metric directly when learning the mapping, as noted above. Recently, \citeauthor{Artetxe2019} \citeyear{Artetxe2019} propose an alternative retrieval method that relies on building a phrase-based MT system from the cross-lingual word embeddings. The MT system is used to generate a synthetic parallel corpus, from which the bilingual lexicon is extracted. The approach has been shown to outperform CSLS retrieval significantly.


\paragraph{Cross-lingual embeddings via retro-fitting} While not strictly a mapping-based approach as it fine-tunes specific monolingual word embeddings, another way to leverage word-level supervision is through the framework of retro-fitting \cite{Faruqui2015}. The main idea behind retro-fitting is to inject knowledge from semantic lexicons into pre-trained word embeddings. Retro-fitting creates a new word embedding matrix $\hat{\mathbf{X}}$ whose embeddings $\hat{\mathbf{x}}_i$ are both close to the corresponding learned monolingual word embeddings $\mathbf{x}_i$ as well as close to their neighbors $\hat{\mathbf{x}}_j$ in a knowledge graph:
\begin{equation}
\Omega_\text{retro} = \sum^{|V|}_{i=1} \Big[\alpha_i \| \hat{\mathbf{x}}_i - \mathbf{x}_i \|^2  + \sum_{(i,j) \in E} \beta_{ij} \| \hat{\mathbf{x}}_i - \hat{\mathbf{x}}_j \|^2 \Big]
\end{equation}
$E$ is the set of edges in the knowledge graph and $\alpha$ and $\beta$ control the strength of the contribution of each term.


While the initial retrofitting work focused solely on monolingual word embeddings \cite{Faruqui2015,Wieting:15}, Mrkšić et al. (2017b) derive both monolingual and cross-lingual synonymy {\em and} antonymy constraints from cross-lingual BabelNet synsets. They then use these constraints to bring the monolingual vector spaces of two different languages together into a shared embedding space. Such retrofitting approaches employ $\Omega_{\text{MMHL}}$ with a careful selection of intruders, similar to the work of \citeauthor{Lazaridou2015} (2015). While through these external constraints retro-fitting methods can capture relations that are more complex than a linear transformation (as with mapping-based approaches), the original post-processing retrofitting approaches are limited to words that are contained in the semantic lexicons, and do not generalise to words unobserved in the external semantic databases. In other words, the goal of retrofitting methods is to refine vectors of words for which additional high-quality lexical information exists in the external resource, while the methods still back off to distributional vector estimates for all other words. 

To remedy the issue with words unobserved in the external resources and learn a global transformation of the entire distributional space in both languages, several methods have been proposed. Post-specialisation approaches first fine-tune vectors of words observed in the external resources, and then aim to learn a global transformation function using the original distributional vectors and their retrofitted counterparts as training pairs. The transformation function can be implemented as a deep feed-forward neural network with non-linear transformations \cite{Vulic:2018post}, or it can be enriched by an adversarial component that tries to distinguish between distributional and retrofitted vectors \cite{Ponti:2018emnlp}. While this is a two-step process (1. retrofitting, 2. global transformation learning), an alternative approach proposed by \cite{Vulic2018} learns a global transformation function directly in one step using external lexical knowledge. Furthermore, \citeauthor{Pandey:ea:17} (2017) explored the orthogonal idea of using cross-lingual word embeddings to transfer the regularization effect of knowledge bases using retrofitting techniques.


%% file: 6-2-Word-alignment-pseudo-bilingual.tex
\subsubsection{Word-level Approaches based on Pseudo-bilingual Corpora}

Rather than learning a mapping between the source and the target language, some approaches use the word-level alignment of a seed bilingual dictionary to construct a pseudo-bilingual corpus by randomly replacing words in a source language corpus with their translations. \citeauthor{Xiao2014} (2014) propose the first such method. 
Using an initial seed lexicon, they create a joint cross-lingual vocabulary, in which each translation pair occupies the same vector representation. They train this model using MML \cite{Collobert2008} by feeding it context windows of both the source and target language corpus.

Other approaches explicitly create a pseudo-bilingual corpus: \citeauthor{Gouws2015a} (2015) concatenate the source and target language corpus and replace each word that is part of a translation pair with its translation equivalent with a probability of $\frac{1}{2k_{t}}$, where $k_t$ is the total number of possible translation equivalents for a word, and train CBOW on this corpus. \citeauthor{Ammar2016} \citeyear{Ammar2016} extend this approach to multiple languages: Using bilingual dictionaries, they determine clusters of synonymous words in different languages. They then concatenate the monolingual corpora of different languages and replace tokens in the same cluster with the cluster ID. Finally, they train SGNS on the concatenated corpus.

\citeauthor{Duong2016} (2016) propose a similar approach. Instead of randomly replacing every word in the corpus with its translation, they replace each center word with a translation on-the-fly during CBOW training. In addition, they handle polysemy explicitly by proposing an EM-inspired method that chooses as replacement the translation $w_i^t$ whose representation is most similar to the sum of the source word representation $\mathbf{x}_i^s$ and the sum of the context embeddings $\mathbf{x}_s^s$ as in Equation \ref{eq:cbow}:

\begin{equation}
w_i^t = \text{argmax}_{w'\in \tau(w_i)} \: \text{cos}(\mathbf{x}_i + \mathbf{x}_s^s, \mathbf{x}') 
\end{equation}
They jointly learn to predict both the words and their appropriate translations using PanLex as the seed bilingual dictionary. PanLex covers around 1,300 language with about 12 million expressions. Consequently, translations are high coverage but often noisy. \citeauthor{Adams2017} (2017) use the same approach for pre-training cross-lingual word embeddings for low-resource language modeling.

As we will show shortly, methods based on pseudo-bilingual corpora optimize a similar objective to the mapping-based methods we have previously discussed. In practice, however, pseudo-bilingual methods are more expensive as they require training cross-lingual word embeddings from scratch based on the concatenation of large monolingual corpora. In contrast, mapping-based approaches are much more computationally efficient as they leverage pretrained monolingual word embeddings, while the mapping can be learned very efficiently. 

%% file: 6-3-Word-alignment-joint.tex
\subsubsection{Joint Models}
While the previous approaches either optimize a set of monolingual losses and then the cross-lingual regularization term (mapping-based approaches) or optimize a monolingual loss and implicitly---via data manipulation---a cross-lingual regularization term, joint models optimize monolingual and cross-lingual objectives at the same time jointly. In what follows, we discuss a few illustrative example models which sparked this sub-line of research. 


\paragraph{Bilingual language model} \citeauthor{Klementiev2012} (2012) cast learning cross-lingual representations as a multi-task learning problem. They jointly optimize a source language and target language model together with a cross-lingual regularization term that encourages words that are often aligned with each other in a parallel corpus to be similar. 
The monolingual objective is the classic LM objective of minimizing the negative log likelihood of the current word $w_i$ given its $C$ previous context words:
\begin{align}
\mathcal{L} = - \log P(w_i \: | \: w_{i-C+1:i-1}) 
\end{align}
For the cross-lingual regularization term, they first obtain an alignment matrix $\mathbf{A}^{s \rightarrow t}$ that indicates how often each source language word was aligned with each target language word from parallel data such as the Europarl corpus \cite{koehn2009statistical}. The cross-lingual regularization term then encourages the representations of source and target language words that are often aligned in $\mathbf{A}^{s \rightarrow t}$ to be similar:
\begin{align}
\Omega_s = -\sum^{|V|^s}_{i=1} \dfrac{1}{2} \mathbf{x}_i^s{}^\top (\mathbf{A}^{s \rightarrow t} \otimes \mathbf{I}_d) \mathbf{x}_i^s
\end{align}
where $\mathbf{I}_d \in \mathbb{R}^{d \times d}$ is the identity matrix and $\otimes$ is the Kronecker product, which intuitively ``blows up'' each element of $\mathbf{A}^{s \rightarrow t} \in \mathbb{R}^{|V^s| \times |V^t|}$ to the size of $\mathbf{x}_i^s \in \mathbb{R}^d$. The final regularization term will be the sum of $\Omega_s$ and the analogous term for the other direction ($\Omega_t$). Note that Equation~(24) is a weighted (by word alignment scores) average of inner products, and hence, for unit length normalized embeddings, equivalent to approaches that maximize the sum of the cosine similarities of aligned word pairs. Using $\mathbf{A}^{s \rightarrow t} \otimes \mathbf{I}$ to encode interaction is borrowed from linear multi-task learning models \cite{cavallanti2010linear}. There, an interaction matrix $\mathbf{A}$ encodes the relatedness between tasks. 
The complete objective is the following:
\begin{equation}
J = \mathcal{L}(\mathbf{X}_s) + \mathcal{L}(\mathbf{X}_t) + \Omega(\underline{\mathbf{A}^{s \rightarrow t}},\mathbf{X}_s) + \Omega(\underline{\mathbf{A}^{t \rightarrow s}},\mathbf{X}_t)
\end{equation}

\paragraph{Joint learning of word embeddings and word alignments}
\citeauthor{Kocisky2014a} (2014) simultaneously learn word embeddings and word-level alignments using a distributed version of FastAlign \cite{Dyer2013} together with a language model.\footnote{FastAlign is a fast and effective variant of IBM Model 2.} Similar to other bilingual approaches, they use the word in the source language sentence of an aligned sentence pair to predict the word in the target language sentence.

They replace the standard  multinomial translation probability of FastAlign with an energy function that tries to bring the representation of a target word $w^t_i$ close to the sum of the context words around the word $w^s_i$ in the source sentence: 
\begin{align}
E(w_i^s, w_i^t, ) = - (\sum\limits_{j=-C}^C \mathbf{x}_{i+j}^s{}^\top \mathbf{T}) \mathbf{x}_i^t - \mathbf{b}^\top \mathbf{x}_i^t - b_{w_i^t}
\end{align}
where $\mathbf{x}_{i+j}^s$ and $\mathbf{x}_i^t$ are the representations of source word $w_{i+j}^s$ and target word $w_i^t$ respectively, $\mathbf{T} \in \mathbb{R}^{d \times d}$ is a projection matrix, and $\mathbf{b} \in \mathbb{R}^d$ and $b_{w_i^t} \in \mathbb{R}$ are representation and word biases respectively. The method is trained via Expectation Maximization.
Note that this model is conceptually very similar to bilingual models that discard word-level alignment information and learn solely on the basis of sentence-aligned information, which we discuss in Section~\ref{sec:sentence_parallel}.


\subsubsection{Sometimes Mapping, Joint and Pseudo-bilingual Approaches are Equivalent} Below we show that while mapping, joint and pseudo-bilingual approaches seem very different, intuitively, they can sometimes be very similar, and in fact, equivalent. We demonstrate this by first defining a pseudo-bilingual approach that is equivalent to an established joint learning technique; and by then showing that same joint learning technique is equivalent to a popular mapping-based approach (for a particular hyper-parameter setting). 

We define {\sc Constrained Bilingual SGNS}. First, recall that in the negative sampling objective of SGNS in Equation \ref{eq:negative_sampling}, the  probability of observing a word $w$ with a context word $c$ with representations $\mathbf{x}$ and $\mathbf{\tilde{x}}$ respectively is given as
$P(c \:|\: w) = \sigma(\mathbf{\tilde{x}}^\top \mathbf{x})$, where $\sigma$ denotes the sigmoid function. We now sample a set of $k$ negative examples, that is, contexts $c_i$ with which $w$ does not occur, as well as actual contexts $C$ consisting of $(w_j, c_j)$ pairs, and try to maximize the above for actual contexts and minimize it for negative samples. Second, recall that \citeauthor{Mikolov2013b} \citeyear{Mikolov2013b} obtain cross-lingual embeddings by running SGNS over two monolingual corpora of two different languages at the same time with the constraint that words known to be translation equivalents, according to some dictionary $D$, have the same representation. We will refer to this as {\sc Constrained Bilingual SGNS}. This is also the approach taken in \citeauthor{Xiao2014}~(2014). $D$ is a function from words $w$ into their translation equivalents $w'$ with the representation $\mathbf{x}'$. With some abuse of notation, we can write the {\sc Constrained Bilingual SGNS} objective for the source language (idem for the target language):

\begin{equation}
\sum_{(w_j, c_j) \in C} \log \sigma(\mathbf{\tilde{x}}_j{}^\top \mathbf{x}_j) +\sum^k_{i=1} \log \sigma(-\mathbf{\tilde{x}}_{i}{}^\top \mathbf{x}_j)\\
+\Omega_\infty \sum_{w'\in \tau(w_j)}|\mathbf{x}_j-\mathbf{x}_j'|
\end{equation}
In pseudo-bilingual approaches, we instead sample sentences from the corpora in the two languages. When we encounter a word $w$ for which we have a translation, that is, $\tau(w)\neq \emptyset$ we flip a coin and if heads, we replace $w$ with a randomly selected member of $D(w)$. In the case, where $D$ is bijective as in the work of \citeauthor{Xiao2014} (2014), it is easy to see that the two approaches are equivalent, in the limit: Sampling mixed $\langle w,c\rangle$-pairs, $w$ and $D(w)$ will converge to the same representations. We can loosen the requirement that $D$ is bijective. 
To see this, assume, for example, the following word-context pairs: $\langle a, b\rangle, \langle a, c\rangle, \langle a, d\rangle$. The vocabulary of our source language is $\{a, b, d\}$, and the vocabulary of our target language is $\{a, c, d\}$. Let $a_s$ denote the source language word in the word pair $a$; etc. To obtain a mixed corpus, such that running SGNS directly on it, will induce the same representations, in the limit, simply enumerate all combinations: $\langle a_s, b\rangle, \langle a_t, b\rangle, \langle a_s, c\rangle, \langle a_t, c\rangle, \langle a_s, d_s\rangle,$ $\langle a_s, d_t\rangle, \langle a_t, d_s\rangle, \langle a_t, d_t\rangle$. Note that this is exactly the mixed corpus you would obtain in the limit with the approach by \citeauthor{Gouws2015a} (2015). Since this reduction generalizes to all examples where $D$ is bijective, this translation provides a constructive demonstration that for any {\sc Constrained Bilingual SGNS} model, there exists a corpus such that pseudo-bilingual sampling learns the same embeddings as this model. In order to complete the demonstration, we need to establish equivalence in the other direction: Since the mixed corpus constructed using the method in \citeauthor{Gouws2015a}~(2015) samples from all replacements licensed by the dictionary, in the limit all words in the dictionary are distributionally similar and will, in the limit, be represented by the same vector representation. This is exactly {\sc Constrained Bilingual SGNS}. It thus follows that:

\begin{lemma} Pseudo-bilingual sampling is, in the limit, equivalent to {\sc Constrained Bilingual SGNS}.\end{lemma}

While mapping-based and joint approaches seem very different at first sight, they can also be very similar---and, in fact, sometimes equivalent. We give an example of this by demonstrating that two methods in the literature are equivalent under some hyper-parameter settings:

Consider the mapping approach in \citeauthor{Faruqui2015}~(2015) ({\em retro-fitting}) and {\sc Constrained Bilingual SGNS} \cite{Xiao2014}. Retro-fitting requires two pretrained monolingual embeddings. Let us assume these embeddings were induced using SGNS with a set of hyper-parameters $\mathcal{Y}$. Retro-fitting minimizes the weighted sum of the Euclidean distances between the seed words and their translation equivalents and their neighbors in the monolingual embeddings, with a parameter $\alpha$ that controls the strength of the regularizer. As this parameter goes to infinity, translation equivalents will be forced to have the same representation. As is the case in {\sc Constrained Bilingual SGNS}, all word pairs in the seed dictionary will be associated with the same vector representation. 

Since retro-fitting only affects words in the seed dictionary, the representation of the words not seen in the seed dictionary is determined entirely by the monolingual objectives. Again, this is the same as in {\sc Constrained Bilingual SGNS}. In other words, if we fix $\mathcal{Y}$ for retro-fitting and {\sc Constrained Bilingual SGNS}, and set the regularization strength $\alpha=\Omega_\infty$ in retro-fitting, retro-fitting is equivalent to 
{\sc Constrained Bilingual SGNS}. 


\begin{lemma}
Retro-fitting of SGNS vector spaces with $\alpha=\Omega_\infty$ is equivalent to {\sc Constrained Bilingual SGNS}.\footnote{All other hyper-parameters are shared and equal, including the dimensionality $d$ of the vector spaces.}\end{lemma}

\begin{proof} 
We provide a simple bidirectional constructive proof, defining a translation function $\tau$ from each retro-fitting model $r_i=\langle S,T,(S,T)/\sim,\alpha=\Omega_\infty\rangle$, with $S$ and $T$ source and target SGNS embeddings, and $(S, T)/\sim$ an equivalence relation between source and target embeddings $(w,\mathbf{w})$, with $\mathbf{w}\in \mathbb{R}^d$, to a {\sc Constrained Bilingual SGNS} model $c_i=\langle S,T,D\rangle$, and back.

 Retro-fitting minimizes the weighted sum of the Euclidean distances between the seed words and their translation equivalents and their neighbors in the monolingual embeddings, with a parameter $\alpha$ that controls the strength of the regularizer. As this parameter goes to infinity ($\alpha\longmapsto\Omega_\infty$), translation equivalents will be forced to have the same representation. In both retro-fitting and {\sc Constrained Bilingual SGNS}, only words in $(S,T)/\sim$ and $D$ are directly affected by regularization; the other words only indirectly by being penalized for not being close to distributionally similar words in $(S,T)/\sim$ and $D$. 
 
We therefore define $\tau(\langle S,T,(S,T)/\sim,\alpha=\Omega_\infty\rangle)=\langle S,T,D \rangle$, s.t., $(s,t)\in D$ iff $(s,t)\in (S,T)/\sim$. Since this function is bijective, $\tau^{-1}$ provides the backward function from {\sc Constrained Bilingual SGNS} models to retro-fitting models. This completes the proof that retro-fitting of SGNS vector spaces and {\sc Constrained Bilingual SGNS} are equivalent when $\alpha=\Omega_\infty$. 
 

\end{proof} 

%% file: 6-4-Word-alignment-comparable.tex
\subsection{Word-Level Alignment Methods with Comparable Data}

All previous methods required word-level \emph{parallel} data. We categorize methods that employ word-level alignment with \emph{comparable} data into two types:

\begin{itemize}
\item[a)] \textbf{Language grounding models} anchor language in images and use image features to obtain information with regard to the similarity of words in different languages.
\item[b)] \textbf{Comparable feature models} that rely on the comparability of some other features. The main feature that has been explored in this context is part-of-speech (POS) tag equivalence. 
\end{itemize}

\paragraph{Grounding language in images} Most methods employing word-aligned comparable data ground words from different languages in image data. The idea in all of these approaches is to use the image space as the shared cross-lingual signals. For instance, bicycles always look like bicycles even if they are referred to as ``fiets'', ``Fahrrad'', ``bicikl'', ``bicicletta'', or ``velo''. A set of images for each word is typically retrieved using Google Image Search. \citeauthor{Bergsma2011} (2011) calculate a similarity score for a pair of words based on the visual similarity of their associated image sets. They propose two strategies to calculate the cosine similarity between the color and \textsc{SIFT} features of two image sets: They either take the average of the maximum similarity scores (\textsc{AvgMax}) or the maximum of the maximum similarity scores (\textsc{MaxMax}). \citeauthor{Kiela2015b} (2015) propose to do the same but use CNN-based image features. \citeauthor{Vulic2016} (2016) in addition propose to combine image and word representations either by interpolating and concatenating them or by interpolating the linguistic and visual similarity scores.

A similar idea of grounding language for learning multimodal multilingual representations can be applied for sensory signals beyond vision, e.g. auditive or olfactory signals \cite{Kiela2015a}. This line of work, however, is currently under-explored. Moreover, it seems that signals from other modalities are typically more useful as an additional source of information to complement the uni-modal signals from text, rather than using other modalities as the single source of information.

\paragraph{POS tag equivalence} Other approaches rely on comparability between certain features of a word, such as its part-of-speech tag. \citeauthor{Gouws2015a} (2015) create a pseudo-cross-lingual corpus by replacing words based on part-of-speech equivalence, that is, words with the same part-of-speech in different languages are replaced with one another. Instead of using the POS tags of the source and target words as a bridge between two languages, we can also use the POS tags of their contexts. This makes strong assumptions about the word orders in the two languages, and their similarity, but \citeauthor{Duong:2015conll} (2015) use this to obtain cross-lingual word embeddings for several language pairs. They use POS tags as context features and run SGNS on the concatenation of two monolingual corpora. Note that under the (too simplistic) assumptions that all instances of a part-of-speech have the same distribution, and each word belongs to a single part-of-speech class, this approach is equivalent to the pseudo-cross-lingual corpus approach described before. 


\paragraph{Summary} Overall, parallel data on the word level is generally preferred over comparable data, as it is relatively easy to obtain for most language pairs and methods relying on parallel data have been shown to outperform methods leveraging comparable data. For methods relying on word-aligned parallel data, even though they optimize similar objectives, mapping-based approaches are the current tool of choice for learning cross-lingual word embeddings due to their conceptual similarity, ease of use, and by virtue of being relatively computationally inexpensive. As monolingual word embeddings have already been learned from large amounts of unlabelled data, the mapping can typically be produced in tens of minutes on a CPU. While unsupervised mapping-based approaches are particularly promising, they still fail for distant language pairs \cite{Soegaard2018} and generally---despite some claims to the contrary---underperform their supervised counterparts \cite{Glavas2019}. At this point, the most robust unsupervised method is the heuristics-based initialisation method by \citeauthor{Artetxe2018b} \citeyear{Artetxe2018b}, while the most robust supervised method is the extension of the Procrustes method with mutual nearest neighbours by \citeauthor{Glavas2019} (2019). We discuss challenges in Section \ref{sec:challenges}.

%% file: 7-Sentence-alignment.tex
\section{Sentence-Level Alignment Methods} \label{sec:sentence_level_alignment_models}

Thanks to research in MT, large amounts of sentence-aligned parallel data are available for European languages, which has led to much work focusing on learning cross-lingual representations from sentence-aligned parallel data. For low-resource languages or new domains, sentence-aligned parallel data is more expensive to obtain than word-aligned data as it requires fine-grained supervision. Only recently have methods started leveraging sentence-aligned comparable data.

\subsection{Sentence-Level Methods with Parallel data}
\label{sec:sentence_parallel}

Methods leveraging sentence-aligned data are generally extensions of successful monolingual models. We have detected four main types:

\begin{itemize}
\item[a)] \textbf{Word-alignment based matrix factorization approaches} apply matrix factorization techniques to the bilingual setting and typically require additional word alignment information.
\item[b)] \textbf{Compositional sentence models} use word representations to construct sentence representations of aligned sentences, which are trained to be close to each other.
\item[c)] \textbf{Bilingual autoencoder models} reconstruct source and target sentences using an autoencoder.
\item[d)] \textbf{Bilingual skip-gram models} use the skip-gram objective to predict words in both source and target sentences. 
\end{itemize}

\paragraph{Word-alignment based matrix factorization} Several methods directly leverage the information contained in an alignment matrix $\mathbf{A}^{s \rightarrow t}$ between source language $s$ and target language $t$ respectively. $\mathbf{A}^{s \rightarrow t}$ is generally automatically derived from sentence-aligned parallel text using an unsupervised word alignment model such as FastAlign \cite{Dyer2013}. $\mathbf{A}^{s \rightarrow t}_{ij}$ captures the number of times the $i$-th word in language $t$ was aligned with the $j$-th word in language $s$, with each row normalized to $1$. The intuition is that if a word in the source language is only aligned with one word in the target language, then those words should have the same representation. If the target word is aligned with more than one source word, then its representation should be a combination of the representations of its aligned words. \citeauthor{Zou2013} (2013) represent the embeddings $\mathbf{X}^s$ in the target language as the product of the source embeddings $\mathbf{X}^s$ with the corresponding alignment matrix $\mathbf{A}^{s \rightarrow t}$. They then minimize the squared difference between these two terms in both directions:
\begin{align}
\Omega_{s\rightarrow t} &= || \mathbf{X}^t - \mathbf{A}^{s \rightarrow t} \mathbf{X}^s||^2 
\end{align}
Note that $\Omega_{s\rightarrow t}$ can be seen as a variant of $\Omega_{\text{MSE}}$, which incorporates soft weights from alignments. In contrast to mapping-based approaches, the alignment matrix, which transforms source to target embeddings, is fixed in this case, while the corresponding source embeddings $\mathbf{X}^s$ are learned:
\begin{equation}
J = \underbrace{\mathcal{L}_{\text{MML}}(\mathbf{X}^t)}_\text{1} + \underbrace{\Omega_{s\rightarrow t}(\underline{\mathbf{X}^t}, \underline{\mathbf{A}^{s\rightarrow t}}, \mathbf{X}^s)}_\text{2}
\end{equation}

\citeauthor{Shi2015} (2015) also take into account monolingual data by placing cross-lingual constraints on the monolingual representations 
and propose two alignment-based cross-lingual regularization objectives. The first one treats the alignment matrix $\mathbf{A}^{s \rightarrow t}$ as a cross-lingual co-occurrence matrix and factorizes it using the GloVe objective. The second one is similar to the objective by \citeauthor{Zou2013} (2013) and minimizes the squared distance of the representations of words in two languages weighted by their alignment probabilities.


\citeauthor{Gardner2015} (2015) extend LSA as translation-invariant LSA to learn cross-lingual word embeddings. They factorize a multilingual co-occurrence matrix with the restriction that it should be invariant to translation, i.e., it should stay the same if multiplied with the respective word or context dictionary.

\citeauthor{Vyas2016} (2016) propose another method based on matrix factorization that enables learning sparse cross-lingual embeddings. As the sparse cross-lingual embeddings are different from the monolingual embeddings $\mathbf{X}$, we diverge slightly from our notation and designate them as $\mathbf{S}$. They propose two constraints: The first constraint induces monolingual sparse representations from pre-trained monolingual embedding matrices $\mathbf{X}^s$ and $\mathbf{X}^t$ by factorizing each embedding matrix $\mathbf{X}$ into two matrices $\mathbf{S}$ and $\mathbf{E}$ with an additional $\ell_1$ constraint on $\mathbf{S}$ for sparsity:
\begin{align}
\mathcal{L} = \sum\limits_{i=1}^{|V|} \|\mathbf{S}_i\mathbf{E}^\top - \mathbf{X}_i\|_2^2 + \lambda \|\mathbf{S}_i\|_1 \end{align}
To learn bilingual embeddings, they add another constraint based on the alignment matrix $\mathbf{A}^{s \rightarrow t}$ that minimizes the $\ell_2$ reconstruction error between words that were strongly aligned to each other in a parallel corpus: 
\begin{align}
\Omega = \sum\limits_{i=1}^{|V^s|} \sum\limits_{j=1}^{|V^t|} \dfrac{1}{2} \lambda_x \mathbf{A}_{ij}^{s \rightarrow t} \|\mathbf{S}^s_i - \mathbf{S}^t_j\|_2^2
\end{align}
The complete optimization then consists of first pre-training monolingual embeddings $\mathbf{X}^s$ and $\mathbf{X}^t$ with GloVe and in a second step factorizing the monolingual embeddings with the cross-lingual constraint to induce cross-lingual sparse representations $\mathbf{S}^s$ and $\mathbf{S}^t$:
\begin{equation}
J = \underbrace{\mathcal{L}_{\text{GloVe}}(\mathbf{X}^s) + \mathcal{L}_{\text{GloVe}}(\mathbf{X}^t)}_\text{1} + \underbrace{\mathcal{L}(\underline{\mathbf{X}^s}, \mathbf{S}^s, \mathbf{E}^s) + \mathcal{L}(\underline{\mathbf{X}^t}, \mathbf{S}^t, \mathbf{E}^t) + \Omega(\underline{\mathbf{A}^{s\rightarrow t}}, \mathbf{S}^s, \mathbf{S}^t)}_\text{2}
\end{equation}

\citeauthor{Guo2015} (2015) similarly create a target language word embedding $\mathbf{x}_i^t$ of a source word $w_i^s$ by taking the average of the embeddings of its translations $\tau(w_i^s)$ weighted by their alignment probability with the source word:
\begin{align}
\mathbf{x}_i^t = \sum \limits_{w^t_j \in \tau(w^s_i)} \dfrac{\mathbf{A}_{i, j}}{\mathbf{A}_{i,\cdot}} \cdot \mathbf{x}_j^t
\end{align}
They propagate alignments to out-of-vocabulary (OOV) words using edit distance as an approximation for morphological similarity and set the target word embedding $\mathbf{x}_k^t$ of an OOV source word $w^s_k$ as the average of the projected vectors of source words that are similar to it based on the edit distance measure:
\begin{align}
\mathbf{x}_k^t = \frac{1}{|E_k|} \sum_{w^s \in E_k} \mathbf{x}^t
\end{align}
where $\mathbf{x}^t$ is the target language word embedding of a source word $w^s$ as created above, $E_k = \{ w^s \:|\: EditDist(w_k^s, w^s) \leq \chi \}$, and $\chi$ is set empirically to $1$.

\paragraph{Compositional sentence model} \citeauthor{Hermann2013} (2013) train a model to bring the sentence representations of aligned sentences $sent^s$ and $sent^t$ in source and target language $s$ and $t$ respectively close to each other. 
The representation $\mathbf{y}^s$ of sentence $sent^s$ in language $s$ is the sum of the embeddings of its words:
\begin{align} \label{eq:comp_sent_model_sent_repr}
\mathbf{y}^s = \sum\limits_{i=1}^{|sent^s|} \mathbf{x}_i^s
\end{align}
They seek to minimize the distance between aligned sentences $sent^s$ and $sent^t$:
\begin{align} \label{eq:comp_sent_model_sent_dist}
E_{dist}(sent^s,sent^t) = \| \mathbf{y}^s - \mathbf{y}^t\|^2
\end{align}
They optimize this distance using MML by learning to bring aligned sentences closer together than randomly sampled negative examples:
\begin{align}
\mathcal{L} = \sum_{(sent^s, sent^t) \in \mathcal{C}}\sum^k_{i=1} \max(0, 1 + E_{dist}(sent^s, sent^t) - E_{dist}(sent^s, s^t_i))
\end{align}
where $k$ is the number of negative examples. In addition, they use an $\ell_2$ regularization term for each language $\Omega = \dfrac{\lambda}{2} \| \mathbf{X} \|^2$ so that the final loss they optimize is the following:
\begin{equation}
J = \mathcal{L}(\mathbf{X}^s, \mathbf{X}^t) + \Omega(\mathbf{X}^s) + \Omega(\mathbf{X}^t)
\end{equation}
Note that compared to most previous approaches, there is no dedicated monolingual objective and all loss terms are optimized jointly. Note that in this case, the $\ell_2$ norm is applied to representations $\mathbf{X}$, which are computed as the difference of sentence representations. 

This regularization term {\em approximates} minimizing the mean squared error between the pair-wise interacting source and target words in a way similar to \citeauthor{Gouws2015}~(2015). To see this, note that we minimize the squared error between source and target representations, i.e. $\Omega_{\text{MSE}}$---this time only not with regard to word embeddings but with regard to sentence representations. As we saw, these sentence representations are just the sum of their constituent word embeddings. In the limit of infinite data, we therefore implicitly optimize $\Omega_{\text{MSE}}$ over word representations. 

\citeauthor{Hermann2014} (2014) extend this approach to documents, by applying the composition and objective function recursively to compose sentences into documents. They additionally propose a non-linear composition function based on bigram pairs, which outperforms simple addition on large training datasets, but underperforms it on smaller data:
\begin{align}
\mathbf{y} = \sum\limits_{i=1}^n \text{tanh}(\mathbf{x}_{i-1} + \mathbf{x}_i)
\end{align}
\citeauthor{Soyer2015} (2015) augment this model with a monolingual objective that operates on the phrase level.
The objective uses MML and is based on the assumption that phrases are typically more similar to their sub-phrases than to randomly sampled phrases:
\begin{align}
\mathcal{L} = [\max(0, m + \|\mathbf{a}_o - \mathbf{a}_i\|^2 - \|\mathbf{a}_o - \mathbf{b}_n\|^2) + \|\mathbf{a}_o - \mathbf{a}_i\|^2] \dfrac{n_i}{n_o}
\end{align}
where $m$ is a margin, $\mathbf{a}_o$ is a phrase of length $n_o$ sampled from a sentence, $\mathbf{a}_i$ is a sub-phrase of $\mathbf{a}_o$ of length $n_i$, and $\mathbf{b}_n$ is a phrase sampled from a random sentence. The additional loss terms are meant to reduce the influence of the margin as a hyperparameter and to compensate for the differences in phrase and sub-phrase length.


\paragraph{Bilingual autoencoder} Instead of minimizing the distance between two sentence representations in different languages, \citeauthor{Lauly2013} (2013) aim to reconstruct the target sentence from the original source sentence. Analogously to \citeauthor{Hermann2013} (2013), they also encode a sentence as the sum of its word embeddings. They then train an auto-encoder with language-specific encoder and decoder layers and hierarchical softmax to reconstruct from each sentence the sentence itself and its translation. In this case, the encoder parameters are the word embedding matrices $\mathbf{X}^s$ and $\mathbf{X}^t$, while the decoder parameters are transformation matrices that project the encoded representation to the output language space. The loss they optimize can be written as follows:
\begin{equation}
J = \mathcal{L}_\text{AUTO}^{s \rightarrow s} + \mathcal{L}_\text{AUTO}^{t \rightarrow t} + \mathcal{L}_\text{AUTO}^{s \rightarrow t} + \mathcal{L}_\text{AUTO}^{t \rightarrow s}
\end{equation}
where $\mathcal{L}_\text{AUTO}^{s \rightarrow t}$ denotes the loss for reconstructing from a sentence in language $s$ to a sentence in language $t$. Aligned sentences are sampled from parallel text and all losses are optimized jointly.

\citeauthor{Chandar2014} (2014) use a binary BOW instead of the hierarchical softmax. 
To address the increase in complexity due to the higher dimensionality of the BOW, they propose to merge the bags-of-words in a mini-batch into a single BOW and to perform updates based on this merged bag-of-words. They also add a term to the objective function that encourages correlation between the source and target sentence representations by summing the scalar correlations between all dimensions of the two vectors.


\paragraph{Bilingual skip-gram} Several authors propose extensions of the monolingual skip-gram with negative sampling (SGNS) model to learn cross-lingual embeddings. We show their similarities and differences in Table \ref{tab:bilingual_skip_gram_comparison}. All of these models jointly optimize monolingual SGNS losses for each language together with one more cross-lingual regularization terms:
\begin{equation}
J = \mathcal{L}_{\text{SGNS}}^{s} + \mathcal{L}_{\text{SGNS}}^{t} + \Omega
\end{equation}
Another commonality is that these models do not require word alignments of aligned sentences. Instead, they make different assumptions about the alignment of the data.
\begin{table}[]
\centering
\resizebox{\textwidth}{!}{%
\begin{tabular}{lccc}
\toprule
Model & Alignment model & Monolingual losses & Cross-lingual regularizer \\
BilBOWA \cite{Gouws2015} & Uniform & $\mathcal{L}_{\text{SGNS}}^s + \mathcal{L}_{\text{SGNS}}^t$ & $\Omega_{\text{BILBOWA}}$ \\
Trans-gram \cite{Coulmance2015} & Uniform & $\mathcal{L}_{\text{SGNS}}^s + \mathcal{L}_{\text{SGNS}}^t$ & $ \Omega_{\text{SGNS}}^{s \rightarrow t} + \Omega_{\text{SGNS}}^{t \rightarrow s}$ \\
BiSkip \cite{Luong2015b} & Monotonic & $\mathcal{L}_{\text{SGNS}}^s + \mathcal{L}_{\text{SGNS}}^t$ & $ \Omega_{\text{SGNS}}^{s \rightarrow t} + \Omega_{\text{SGNS}}^{t \rightarrow s}$  \\
\bottomrule
\end{tabular}%
}
\caption{A comparison of similarities and differences of the three bilingual skip-gram variants.}
\label{tab:bilingual_skip_gram_comparison}
\end{table}
Bilingual Bag-of-Words without Word Alignments \cite<BilBOWA;>{Gouws2015} assumes each word in a source sentence is aligned with \emph{every} word in the target sentence. If we knew the word alignments, we would try to bring the embeddings of aligned words in source and target sentences close together. Instead, under a uniform alignment model which perfectly matches the intuition behind the simplest (lexical) word alignment IBM Model 1 \cite{Brown:1993cl}, we try to bring the \emph{average} alignment close together. In other words, we use the means of the word embeddings in a sentence as the sentence representations $\mathbf{y}$ and seek to minimize the distance between aligned sentence representations:
\begin{align}
\mathbf{y}^s & = \dfrac{1}{|sent^s|} \sum\limits_{i=1}^{|sent^s|} \mathbf{x}_i^s \\
\Omega_{\text{BILBOWA}} & = \sum_{(sent^s, sent^t) \in \mathcal{C}} \| \mathbf{y}^s - \mathbf{y}^t \|^2
\end{align}
Note that this regularization term is very similar to the objective used in the compositional sentence model \cite[Equations \ref{eq:comp_sent_model_sent_repr} and \ref{eq:comp_sent_model_sent_dist}]{Hermann2013}; the only difference is that we use the mean rather than the sum of word embeddings as sentence representations.

Trans-gram \cite{Coulmance2015} also assumes uniform alignment but 
uses the SGNS objective as cross-lingual regularization term. Recall that skip-gram with negative sampling seeks to train a model to distinguish context words from negative samples drawn from a noise distribution based on a center word. In the cross-lingual case, we aim to predict words in the aligned target language sentence based on words in the source sentence. Under uniform alignment, we aim to predict \emph{all} words in the target sentence based on each word in the source sentence:
\begin{equation}
\Omega_{\text{SGNS}}^{s\rightarrow t} = - \sum_{(sent^s, sent^t) \in \mathcal{C}} \dfrac{1}{|sent^s|} \sum_{t=1}^{|sent^s|} \sum_{j=1}^{|sent^t|} \text{log} \: P(w_{t+j} \:|\: w_t)
\end{equation}
where $P(w_{t+j} \:|\: w_t)$ is computed via negative sampling as in Equation \ref{eq:negative_sampling}.

BiSkip \cite{Luong2015b} uses the same cross-lingual regularization terms as Trans-gram, but only aims to predict monotonically aligned target language words: Each source word at position $i$ in the source sentence $sent^s$ is aligned to the target word at position $i \cdot \frac{|sent^s|}{|sent^t|}$ in the target sentence $sent^t$. In practice, all these bilingual skip-gram models are trained by sampling a pair of aligned sentences from a parallel corpus and minimizing for the source and target language sentence the respective loss terms.


In a similar vein, \citeauthor{Pham2015} (2015) propose an extension of paragraph vectors \cite{Le2014a} to the multilingual setting by forcing aligned sentences of different languages to share the same vector representation.

\paragraph{Other sentence-level approaches}
\citeauthor{Levy2017} (2017) use another sentence-level bilingual signal: IDs of the aligned sentence pairs in a parallel corpus. Their model provides a strong baseline for cross-lingual embeddings that is inspired by the Dice aligner commonly used for producing word alignments for MT. Observing that sentence IDs are already a powerful bilingual signal, they propose to apply SGNS to the word-sentence ID matrix. They show that this method can be seen as a generalization of the Dice Coefficient.

\citeauthor{Rajendran2016} (2016) propose a method that exploits the idea of using pivot languages, also tackled in previous work, e.g., \citeauthor{Shezaf2010} (2010). The model requires parallel data between each language and a pivot language and is able to learn a shared embedding space for two languages without any direct alignment signals as the alignment is implicitly learned via their alignment with the pivot language. The model optimizes a correlation term with neural network encoders and decoders that is similar to the objective of the CCA-based approaches \cite{Faruqui2014,Lu:2015naacl}. We discuss the importance of pivoting for learning multilingual word embeddings later in Section~\ref{sec:multilingual_training}.

In practice, sentence-level supervision is a lot more expensive to obtain than word-level supervision, which is available in the form of bilingual lexicons even for many low-resource languages. For this reason, recent work has largely focused on word-level supervised approaches for learning cross-lingual embeddings. Nevertheless, word-level supervision only enables learning cross-lingual word representations, while for more complex tasks we are often interested in cross-lingual sentence representations.

Recently, fueled by work on pretrained language models \cite{Howard2018,Devlin2018}, there have been several extensions of language models to the massively cross-lingual setting, learning cross-lingual representations for many languages at once. \citeauthor{Artetxe2018e} \citeyear{Artetxe2018e} train a BiLSTM encoder with a shared vocabulary on parallel data of many languages. \citeauthor{Lample2019} \citeyear{Lample2019} extend the bilingual skip-gram approach to the masked language modelling \cite{Devlin2018}: instead of predicting words in the source and target language sentences via skip-gram, they predict randomly masked words in both sentences with a deep language model. Alternatively, their approach can also be trained without parallel data only on concatenated monolingual datasets. Similar to results for word-level alignment-based methods \cite{Soegaard2018}, the weak supervision induced by sharing the vocabulary between languages is strong enough as an inductive bias to learn useful cross-lingual representations.

\subsection{Sentence Alignment with Comparable Data}


\paragraph{Grounding language in images} Similarly to approaches based on word-level aligned comparable data, methods that learn cross-lingual representations using sentence alignment with comparable data do so by associating sentences with images \cite{Kadar:2017arxiv}. The associated image captions/annotations can be direct translations of each other, but are not expected to be in general. The images are then used as pivots to induce a shared multimodal embedding space. These approaches typically use Multi30k \cite{Elliott2016}, a multilingual extension of the Flickr30k dataset \cite{Young2014}, which contains 30k images and 5 English sentence descriptions and their German translations for each image. \citeauthor{Calixto2017} (2017) represent images using features from a pre-trained CNN and model sentences using a GRU. They then use MML to assign a higher score to image-description pairs compared to images with a random description.
\citeauthor{Gella2017} (2017) augment this objective with another MML term that also brings the representations of translated descriptions closer together, thus effectively combining learning signals from both visual and textual modality.

\paragraph{Summary} While recent work has almost exclusively focused on word-level mapping-based approaches, recent language model-based approaches \cite{Artetxe2018e,Lample2019} have started to incorporate parallel resources. Mapping-based approaches that have been shown to rely on the assumption that the embedding spaces in two languages are approximately isomorphic struggle when mapping between a high-resource and a more distant low-resource language \cite{Soegaard2018}. As research increasingly considers this more realistic setting, the additional supervision and context provided by sentence-level alignment may prove to be a valuable resource---and complementary to word-level alignment. Early results in this direction indicate that joint training on parallel corpora yields embeddings that are more isomorphic and less sensitive to hubness than mapping-based approaches \cite{Ormazabal2019}. Consequently, we expect a resurgence of interest in sentence-level alignment methods.

%% file: 8-Document-alignment.tex
\section{Document-Level Alignment Models} \label{sec:document_level_alignment_models}

Models that require parallel document alignment presuppose that sentence-level parallel alignment is also present. Such models thus reduce to parallel sentence-level alignment methods, which have been discussed in the previous section. Comparable document-level alignment, on the other hand, is more appealing as it is often cheaper to obtain. Existing approaches generally use Wikipedia documents, which they either automatically align, or they employ already theme-aligned Wikipedia documents discussing similar topics.

\subsection{Document Alignment with Comparable Data}

We divide models using document alignment with comparable data into three types, some of which employ similar general ideas to previously discussed word and sentence-level parallel alignment models:

\begin{itemize}
\item[a)] \textbf{Approaches based on pseudo-bilingual document-aligned corpora} automatically construct a pseudo-bilingual corpus containing words from the source and target language by mixing words from document-aligned documents.
\item[b)] \textbf{Concept-based methods} leverage information about the distribution of latent topics or concepts in document-aligned data to represent words.
\item[c)] \textbf{Extensions of sentence-aligned models} extend methods using sentence-aligned parallel data to also work without parallel data. 
\end{itemize}

\paragraph{Pseudo-bilingual document-aligned corpora} The approach of \citeauthor{VulicMoens2016} (2016) is similar to the pseudo-bilingual corpora approaches discussed in Section \ref{sec:word_level_alignment_models}. In contrast to previous methods, they propose a \emph{Merge and Shuffle} strategy to merge two aligned documents of different languages into a pseudo-bilingual document. This is done by concatenating the documents and then randomly shuffling them by permuting words. The intuition is that as most methods rely on learning word embeddings based on their context, shuffling the documents will lead to robust bilingual contexts for each word. As the shuffling step is completely random, it might lead to sub-optimal configurations.

For this reason, they propose another strategy for merging the two aligned documents, called \emph{Length-Ratio Shuffle}. It assumes that the structures of the two documents are similar: words are inserted into the pseudo-bilingual document by alternating between the source and the target document relying on the order in which they appear in their monolingual document and based on the monolingual documents' length ratio.


\paragraph{Concept-based models} Some methods for learning cross-lingual word embeddings leverage the insight that words in different languages are similar if they are used to talk about or evoke the same multilingual concepts or topics. \citeauthor{Vulic2013a} \citeyear{Vulic2013a} base their method on the cognitive theory of semantic word responses. Their method centers on the intuition that words in source and target language are similar if they are likely to generate similar words as their top semantic word responses. They utilise a probabilistic multilingual topic model again trained on aligned Wikipedia documents to learn and quantify semantic word responses. The embedding $\mathbf{x}_i^s \in \mathbb{R}^{|V^s| + |V^t|} $ of source word $w_i$ is the following vector:
\begin{align}
\mathbf{x}_i^s = [P(w_1^s|w_i), \ldots, P(w_{|V^s|}^s|w_i), P(w_1^t|w_i) \ldots, P(w_{|V^t|}^t|w_i)]
\end{align}
where $[\cdot, \cdot]$ represents concatenation and $P(w_j|w_i)$ is the probability of $w_j$ given $w_i$ under the induced bilingual topic model. The sparse representations may be turned into dense vectors by factorizing the constructed word-response matrix.


\citeauthor{Sogaard2015} (2015) propose an approach that relies on the structure of Wikipedia itself. Their method is based on the intuition that similar words are used to describe the same concepts across different languages. Instead of representing every Wikipedia concept with the terms that are used to describe it, they use an inverted index and represent a word by the concepts it is used to describe. As a post-processing step, dimensionality reduction on the produced word representations in the word-concept matrix is performed. A very similar model by \cite{Vulic:2011acl} uses a bilingual topic model to perform the dimensionality reduction step and learns a shared cross-lingual topical space.

\paragraph{Extensions of sentence-alignment models} \citeauthor{Mogadala2016} (2016) extend the approach of \citeauthor{Pham2015} (2015) to also work without parallel data and adjust the regularization term $\Omega$ based on the nature of the training corpus. Similar to previous work \cite{Hermann2013,Gouws2015}, they use the mean of the word embeddings of a sentence as the sentence representation $\mathbf{y}$ and constrain these to be close together. In addition, they propose to constrain the sentence paragraph vectors $\mathbf{p}^s$ and $\mathbf{p}^t$ of aligned sentences $sent^s$ and $sent^t$ to be close to each other. These vectors are learned via paragraph vectors \cite{Le2014a} for each sentence and stored in embedding matrices $\mathbf{P}^s$ and $\mathbf{P}^t$. The complete regularizer then uses elastic net regularization to combine both terms:
\begin{equation}
\Omega = \sum_{(sent^s, sent^t) \in \mathcal{C}} \alpha ||\mathbf{p}^s - \mathbf{p}^t||^2 + (1-\alpha) \| \mathbf{y}^s - \mathbf{y}^t \|^2
\end{equation}
The monolingual paragraph vector objectives $\mathcal{L}_{\text{SGNS-P}}$ are then optimized jointly with the cross-lingual regularization term:
\begin{equation}
J = \mathcal{L}_{\text{SGNS-P}}^s(\mathbf{P}^s, \mathbf{X}^s) + \mathcal{L}_{\text{SGNS-P}}^t(\mathbf{P}^t, \mathbf{X}^t) + \Omega(\mathbf{P}^s, \mathbf{P}^t, \mathbf{X}^s, \mathbf{X}^t)
\end{equation}
To leverage data that is not sentence-aligned, but where an alignment is still present on the document level, they propose a two-step approach: They use Procrustes analysis \cite{schonemann1966procrustes}, a method for statistical shape analysis, to find the most similar document in language $t$ for each document in language $s$. This is done by first learning monolingual representations of the documents in each language using paragraph vectors on each corpus. Subsequently, Procrustes analysis aims to learn a transformation between the two vector spaces by translating, rotating, and scaling the embeddings in the first space until they most closely align to the document representations in the second space. 
In the second step, they then simply use the previously described method to learn cross-lingual word representations from the alignment documents, this time treating the entire documents as paragraphs.

\paragraph{Summary} So far, document-level alignment has only been shown to provide little additional information compared to sentence-level alignment-based methods and interest in such methods has subsided, together with work on sentence-level alignment methods. Most document-level alignment methods are derived from their sentence-level counterparts. As research turns again to leveraging such methods for jointly learning cross-lingual embeddings, we expect that document-level information will also be considered as a signal. In addition, as tasks such as question answering where models learn representations of documents become more popular, learning cross-lingual document representation will likely become an active area of research once suitable cross-lingual benchmarks are available.

\begin{table}[]
\centering
\resizebox{\textwidth}{!}{%
\begin{tabular}{l c c c l}
\toprule Approach & Monolingual & Regularizer & Joint? & Description \\
\midrule 
Klementiev et al. (2012) & $ \mathcal{L}_{\text{MLE}}$ & $\Omega_{\text{MSE}} $ & \checkmark & Joint\\
Mikolov et al. (2013b) & $\mathcal{L}_{\text{SGNS}}$ & $\Omega_{\text{MSE}} $ & & Projection-based \\
Zou et al. (2013) & $\mathcal{L}_{\text{MMHL}}$ & $\Omega_{\text{MSE}} $ & & Matrix factorization\\
Hermann and Blunsom (2013) & $\mathcal{L}_{\text{MMHL}}$ & $\Omega_{\text{MSE}}^* $ & \checkmark & Sentence-level, joint\\
{Hermann and Blunsom (2014)} & {$\mathcal{L}_{\text{MMHL}}$} & $\Omega_{\text{MSE}}^* $ & \checkmark & Sentence-level + bigram composition\\
{Soyer et al. (2015)} & {$\mathcal{L}_{\text{MMHL}}$} & $\Omega_{\text{MSE}}^* $ & \checkmark & Phrase-level\\
Shi et al. (2015) & $\mathcal{L}_{\text{MMHL}}$ & $\Omega_{\text{MSE}} $ & & Matrix factorization\\
Dinu et al. (2015) & $\mathcal{L}_{\text{SGNS}}$ & $\Omega_{\text{MSE}} $ & & Better neighbour retrieval\\
Gouws et al. (2015) & $\mathcal{L}_{\text{SGNS}}$ & $\Omega_{\text{MSE}}$ & \checkmark & Sentence-level\\
Vyas and Carpuat (2016) & $\mathcal{L}_{\text{GloVe}}$ & $\Omega_{\text{MSE}} $ & & Sparse matrix factorization \\
Hauer et al. (2017) & $\mathcal{L}_{\text{SGNS}}$ & $\Omega_{\text{MSE}}$ & & Cognates\\
Mogadala and Rettinger (2016) & $\mathcal{L}_{\text{SGNS-P}}$ & $\Omega_{\text{MSE}}$ & \checkmark & Elastic net, Procrustes analysis\\
\midrule 

Xing et al. (2015) & $\mathcal{L}_{\text{SGNS}}$ & $\Omega_{\text{MSE}} \: \text{s.t.} \: \mathbf{W}^\top\mathbf{W} = \mathbf{I} $ & & Normalization, orthogonality \\
Zhang et al. (2016b) & $\mathcal{L}_{\text{SGNS}}$ & $\Omega_{\text{MSE}} \: \text{s.t.} \: \mathbf{W}^\top\mathbf{W} = \mathbf{I} $ & & Orthogonality constraint\\
\multirow{2}{*}{Artexte et al. (2016)} & \multirow{2}{*}{$\mathcal{L}_{\text{SGNS}}$} & \multirow{2}{*}{$\Omega_{\text{MSE}} \: \text{s.t.} \: \mathbf{W}^\top\mathbf{W} = \mathbf{I} $} & & Normalization, orthogonality,\\
& & & & mean centering \\
\multirow{2}{*}{Smith et al. (2017)} & \multirow{2}{*}{$\mathcal{L}_{\text{SGNS}}$} & \multirow{2}{*}{$\Omega_{\text{MSE}} \: \text{s.t.} \: \mathbf{W}^\top\mathbf{W} = \mathbf{I} $} & & Orthogonality, inverted softmax \\
& & & & identical character strings \\
\multirow{2}{*}{Artexte et al. (2017)} & \multirow{2}{*}{$\mathcal{L}_{\text{SGNS}}$} & \multirow{2}{*}{$\Omega_{\text{MSE}} \: \text{s.t.} \: \mathbf{W}^\top\mathbf{W} = \mathbf{I} $} & & Normalization, orthogonality,\\
& & & & mean centering, bootstrapping \\
\midrule

Lazaridou et al. (2015) & $\mathcal{L}_{\text{CBOW}}$ & $\Omega_{\text{MMHL}}$ & & Max-margin with intruders\\
Mrkšić et al. (2017b) & $\mathcal{L}_{\text{SGNS}}$ & $\Omega_{\text{MMHL}}$ & & Semantic specialization\\
Calixto et al. (2017) & $\mathcal{L}_{\text{RNN}}$&$\Omega_{\text{MMHL}}$ & \checkmark & Image-caption pairs\\
Gella et al. (2017) & $\mathcal{L}_{\text{RNN}}$&$\Omega_{\text{MMHL}}$ & \checkmark & Image-caption pairs\\
\midrule

Faruqui and Dyer (2014) & $\mathcal{L}_{\text{LSA}}$ & $\Omega_{\text{CCA}}$ & & - \\
Lu et al. (2015) & $\mathcal{L}_{\text{LSA}}$ & $\Omega_{\text{CCA}}$ & & Neural CCA \\
Rajendran et al. (2016) & $\mathcal{L}_{\text{AUTO}}$ & $\Omega_{\text{CCA}}$ & & Pivots\\
Ammar et al. (2016b) & $\mathcal{L}_{\text{LSA}}$ & $\Omega_{\text{CCA}}$ & & Multilingual CCA\\
\midrule

Søgaard et al. (2015) & - & $\Omega_{\text{SVD}}$ & \checkmark & Inverted indexing\\
Levy et al. (2017) & $\mathcal{L}_{\text{PMI}}$ & $\Omega_{\text{SVD}}$ & \checkmark & \\
Levy et al. (2017) & - & $\Omega_{\text{SGNS}}$ & \checkmark & Inverted indexing\\
\midrule

Lauly et al. (2013) & $\mathcal{L}_{\text{AUTO}}$ & $\Omega_{\text{AUTO}}$ & \checkmark & Autoencoder \\
Chandar et al. (2014) & $\mathcal{L}_{\text{AUTO}}$ & $\Omega_{\text{AUTO}}$ & \checkmark & Autoencoder\\

\midrule

Vulić and Moens (2013a) & $\mathcal{L}_{\text{LDA}}$ & $\Omega_\infty^*$ & \checkmark & Document-level\\
Vulić and Moens (2014) & $\mathcal{L}_{\text{LDA}}$ & $\Omega_\infty^*$ & \checkmark & Document-level\\
Xiao and Guo (2014) & $\mathcal{L}_{\text{MMHL}}$ & $\Omega_\infty$ & \checkmark & Pseudo-multilingual\\
Gouws and Søgaard (2015) & $\mathcal{L}_{\text{CBOW}}$ & $\Omega_\infty^*$ & \checkmark & Pseudo-multilingual\\
Luong et al. (2015) & $\mathcal{L}_{\text{SGNS}}$ & $\Omega_\infty^*$  & & Monotonic alignment \\
Gardner et al. (2015) & $\mathcal{L}_{\text{LSA}}$ & $\Omega_\infty^*$
& & Matrix factorization\\
Pham et al. (2015) & $\mathcal{L}_{\text{SGNS-P}}$ & $\Omega_\infty$ & \checkmark & Paragraph vectors\\
Guo et al. (2015) & $\mathcal{L}_{\text{CBOW}}$ &  $\Omega_\infty$ & & Weighted by word alignments\\
Coulmance et al. (2015) & $\mathcal{L}_{\text{SGNS}}$ & $\Omega_\infty^*$ & \checkmark & Sentence-level \\
Ammar et al. (2016a) & $\mathcal{L}_{\text{SGNS}}$ & $\Omega_\infty$ & \checkmark & Pseudo-multilingual\\
Vulić and Korhonen (2016) & $\mathcal{L}_{\text{SGNS}}$ & $\Omega_\infty$ & & Highly reliable seed entries\\
Duong et al. (2016) & $\mathcal{L}_{\text{CBOW}}$ & $\Omega_\infty$ & \checkmark & Pseudo-multilingual, polysemy\\
Vulić and Moens (2016) & $\mathcal{L}_{\text{SGNS}}$ & $\Omega_\infty$ & \checkmark & Pseudo-multilingual documents \\
Adams et al. (2017) & $\mathcal{L}_{\text{CBOW}}$ & $\Omega_\infty$ & \checkmark & Pseudo-multilingual, polysemy\\
\midrule

Bergsma and Van Durme (2011) & - & - & \checkmark & {\sc SIFT} image features, similarity\\
Kiela et al. (2015) & - & - & \checkmark & CNN image features, similarity\\
Vulić et al. (2016) & - & - & \checkmark & CNN features, similarity, interpolation\\
\midrule

Gouws and Søgaard (2015) & $\mathcal{L}_{\text{CBOW}}$ & POS-level $\Omega_\infty^*$ & \checkmark & Pseudo-multilingual\\
Duong et al. (2015) & $\mathcal{L}_{\text{CBOW}}$ & POS-level $\Omega_\infty^*$ & \checkmark & Pseudo-multilingual\\
\bottomrule
\end{tabular}%
}
\caption{Overview of approaches with monolingual objectives and regularization terms, with an indication whether the order of optimization matters and short descriptions. $\Omega_\infty$ represents an infinitely strong regularizer that enforces equality between representations. ${}^*$ implies that the regularization is achieved in the limit.}
\label{tab:overview_objectives}
\end{table}

As a final overview, we list all approaches with their monolingual objectives and regularization terms in Table~\ref{tab:overview_objectives}. The table is meant to reveal the \emph{high-level} objectives and losses each model is optimizing. It also indicates for each method whether all objectives are jointly optimized; if they are, both monolingual losses and regularization terms are optimized jointly; otherwise the monolingual losses are optimized first and the monolingual variables are frozen, while the cross-lingual regularization constraint is optimized. The table obscures smaller differences and implementation details, which can be found in the corresponding sections of this survey or by consulting the original papers. We use $\Omega_\infty$ to represent an infinitely stronger regularizer that enforces equality between representations. Regularizers with a ${}^*$ imply that the regularization is achieved in the limit, e.g. in the pseudo-bilingual case, where examples are randomly sampled with some equivalence, we obtain the same representation in the limit, without strictly enforcing it to be the same representation.

As we have demonstrated, most approaches can be seen as optimizing a combination of monolingual losses with a regularization term. As we can see, some approaches do not employ a regularization term; notably, a small number of approaches, i.e., those that ground language in images, do not optimize a loss but rather use pre-trained image features and a set of similarity heuristics to retrieve translations.

%% file: 9-Multilingual.tex
\section{From Bilingual to Multilingual Training} \label{sec:multilingual_training}
So far, for the sake of simplicity and brevity of presentation, we have put focus on models which induce cross-lingual word embeddings in a shared space comprising only two languages. This standard bilingual setup is also in the focus of almost all research in the field of cross-lingual embedding learning. However, notable exceptions such as the recent work of \citeauthor{Levy2017} (2017) and \citeauthor{Duong:2017eacl} (2017) demonstrate that there are clear benefits to extending the learning process from bilingual to \textit{multilingual} settings, with improved results reported on standard tasks such as word similarity prediction, bilingual dictionary induction, document classification and dependency parsing.

The usefulness of multilingual training for NLP is already discussed by, e.g., \citeauthor{Naseem:2009jair} (2009) and \citeauthor{Snyder:2010icml} (2010). They corroborate a hypothesis that ``variations in ambiguity'' may be used as a form of naturally occurring supervision. In simple words, what one language leaves implicit, another defines explicitly and the target language is thus useful for resolving disambiguation in the source language \cite{Faruqui2014}. While this hypothesis is already true for bilingual settings, using additional languages introduces additional supervision signals which in turn leads to better word embedding estimates (Mrkšić et al., 2017b).

In most of the literature focused on bilingual settings, English is typically on one side, owing its wide use to the wealth of both monolingual resources available for English as well as bilingual resources, where English is paired with many other languages. However, one would ideally want to also exploit cross-lingual links between other language pairs, reaching beyond English. For instance, typologically/structurally more similar languages such as Finnish and Estonian are excellent candidates for transfer learning. Yet, only few readily available parallel resources exist between Finnish and Estonian that could facilitate a direct induction of a shared bilingual embedding space in these two languages.

A multilingual embedding model which maps Finnish and Estonian to the same embedding space shared with English (i.e., English is used as a resource-rich \textit{pivot language}) would also enable exploring and utilizing links between Finnish and Estonian lexical items in the space \cite{Duong:2017eacl}. Further, multilingual shared embedding spaces enable multi-source learning and multi-source transfers: this results in a more general model and is less prone to data sparseness \cite{McDonald:2011emnlp,Agic:2016tacl,Guo:2016aaai,Zoph:2016naacl,Firat:2016naacl}. 
\begin{figure*}[t]
    \centering
    \begin{subfigure}[t]{0.45\linewidth}
        \centering
        \includegraphics[width=1.0\linewidth]{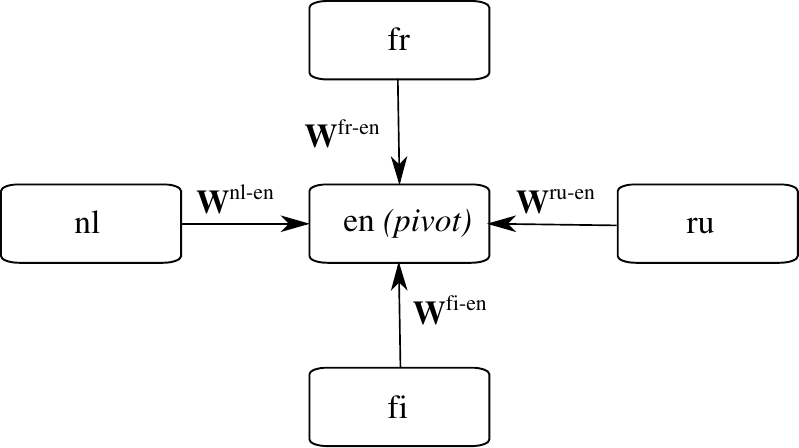}
        \caption{Starting spaces: monolingual}
        \label{fig:multimap1}
    \end{subfigure}
    \begin{subfigure}[t]{0.45\textwidth}
        \centering
        \includegraphics[width=0.98\linewidth]{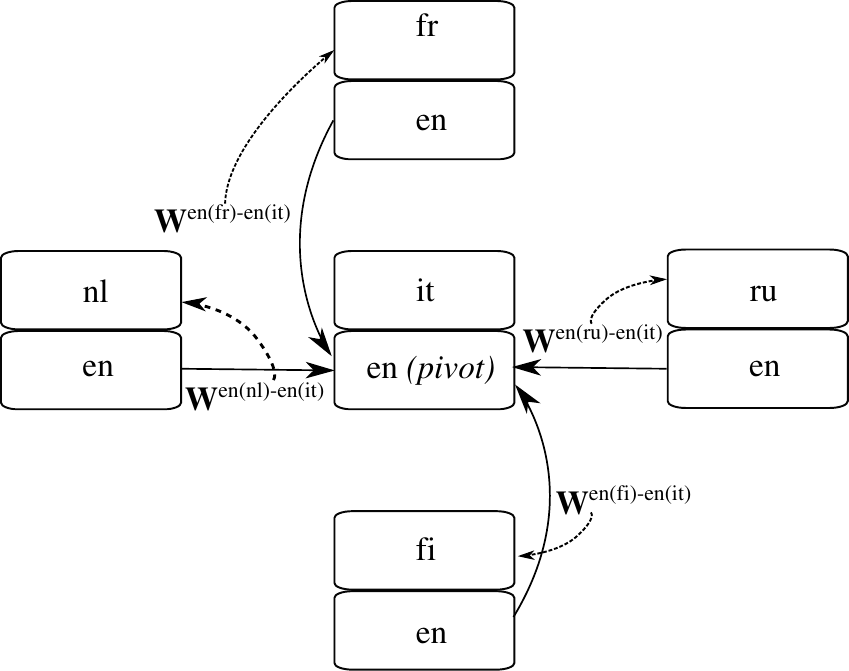}
        \caption{Starting spaces: bilingual}
        \label{fig:multimap2}
    \end{subfigure}
    \vspace{-0.0em}
    \caption{Learning shared multilingual embedding spaces via linear mapping. (a) Starting from monolingual spaces in $L$ languages, one linearly maps $L-1$ into one chosen pivot monolingual space (typically English); (b) Starting from bilingual spaces sharing a language (typically English), one learns mappings from all other English subspaces into one chosen pivot English subspace and then applies the mapping to all other subspaces.}
\vspace{-0.3em}
\label{fig:multimap}
\end{figure*}

The purpose of this section is not to demonstrate the multilingual extension of every single bilingual model discussed in previous sections, as these extensions are not always straightforward and include additional modeling work. However, we will briefly present and discuss several best practices and multilingual embedding models already available in the literature, again following the typology of models established in Table~\ref{tab:cross-lingual_embedding_models}.

\subsection{Multilingual Word Embeddings from Word-Level Information}
\label{ss:multi_word_level}
\paragraph{Mapping-based approaches} 
Mapping $L$ different monolingual spaces into a single multilingual shared space is achieved by: (1) selecting one space as the \textit{pivot space}, and then (2) mapping the remaining $L-1$ spaces into the same {pivot space}. This approach, illustrated by Figure~\ref{fig:multimap1}, requires $L$ monolingual spaces and $L-1$ seed translation dictionaries to achieve the mapping. Labeling the pivot language as $l^p$, we can formulate the induction of a multilingual embedding space as follows:

\begin{align}
\mathcal{L}^1 + \mathcal{L}^2 + \ldots + \mathcal{L}^{L-1} + \mathcal{L}^{p} + \Omega^{l^1\rightarrow l^p} + \Omega^{l^2\rightarrow l^p} + \ldots + \Omega^{l^{L-1}\rightarrow l^p} 
\end{align}
This means that through \textit{pivoting} one is able to induce a shared bilingual space for a language pair without having any directly usable bilingual resources for the pair. Exactly this multilingual mapping procedure (based on minimizing mean squared errors) has been constructed by \citeauthor{Smith2017} (2017): English is naturally selected as the pivot, and 89 other languages are then mapped into the pivot English space. Seed translation pairs are obtained through Google Translate API by translating the 5,000 most frequent words in each language to English. The recent work of, e.g., \citeauthor{Artetxe:2017acl} (2017) holds promise that seed lexicons of similar sizes may also be bootstrapped for resource-lean languages from very small seed lexicons (see again Section~\ref{sec:word_level_alignment_models}). \citeauthor{Smith2017} use original fastText vectors available in 90 languages \cite{Bojanowski:2017tacl}\footnote{The latest release of fastText vectors contains vectors for 204 languages. The vectors are available here: \\ \texttt{https://github.com/facebookresearch/fastText}} and effectively construct a multilingual embedding space spanning 90 languages (i.e., 4005 language pairs using 89 seed translation dictionaries) in their software and experiments.\footnote{\texttt{https://github.com/Babylonpartners/fastText\_multilingual}} The distances in all monolingual spaces remain preserved by constraining the transformation to be orthogonal.
\begin{figure*}[t]
    \centering
        \includegraphics[width=0.55\linewidth]{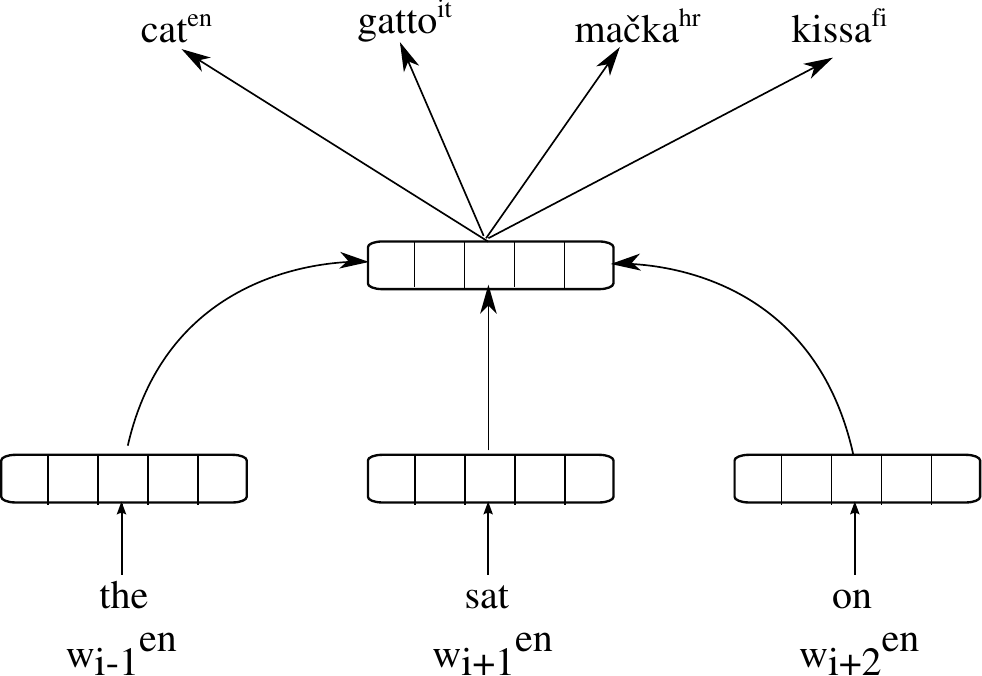}
        \caption{Illustration of the joint multilingual model of Duong et al. (2017) based on the modified CBOW objective; instead of predicting only the English word given the English context, the model also tries to predict its translations in all the remaining languages (i.e., in languages for which the translations exist in any of the input bilingual lexicons).}
        \label{fig:multicbow}
\end{figure*}

Along the same line, \citeauthor{Ammar2016} \citeyear{Ammar2016} introduce a multilingual extension of the CCA-based mapping approach. They perform a multilingual extension of bilingual CCA projection again using the English embedding space as the pivot for multiple \textit{English}-$l^t$ bilingual CCA projections with the remaining $L-1$ languages.

As demonstrated by Smith et al. \citeyear{Smith2017}, the multilingual space now enables reasoning for language pairs not represented in the seed lexicon data. They verify this hypothesis by examining the bilingual lexicon induction task for all $L\choose 2$ language pairs: e.g., BLI Precision-at-one ($P@1$) scores\footnote{$P@1$ is a standard evaluation measure for bilingual lexicon induction that refers to the proportion of source test words for which the best translation is ranked as the most similar word in the target language.} for Spanish-Catalan without any seed Spanish-Catalan lexicon are 0.82, while the average $P@1$ score for Spanish-English and Catalan-English bilingual spaces is 0.70. Other striking findings include $P@1$ scores for Russian-Ukrainian (0.84 vs. 0.59), Czech-Slovak (0.82 vs. 0.59), Serbian-Croatian (0.78 vs. 0.56), or Danish-Norwegian (0.73 vs. 0.67).

A similar approach to constructing a multilingual embedding space is discussed by Duong et al. \citeyear{Duong:2017eacl}. However, their mapping approach is tailored for another scenario frequently encountered in practice: one has to align bilingual embedding spaces where English is one of the two languages in each bilingual space. In other words, our starting embedding spaces are now not monolingual as in the previous mapping approach, but bilingual. The overview of the approach is given in Figure~\ref{fig:multimap2}. This approach first selects a pivot bilingual space (e.g., this is the EN-IT space in Figure~\ref{fig:multimap2}), and then learns a linear mapping/transformation from the English subspace of all other bilingual spaces into the pivot English subspace. The learned linear mapping is then applied to other subspaces (i.e., ``foreign'' subspaces such as FI, FR, NL, or RU in Figure~\ref{fig:multimap2}) to transform them into the shared multilingual space. 

\begin{figure*}[t]
    \centering
        \includegraphics[width=0.91\linewidth]{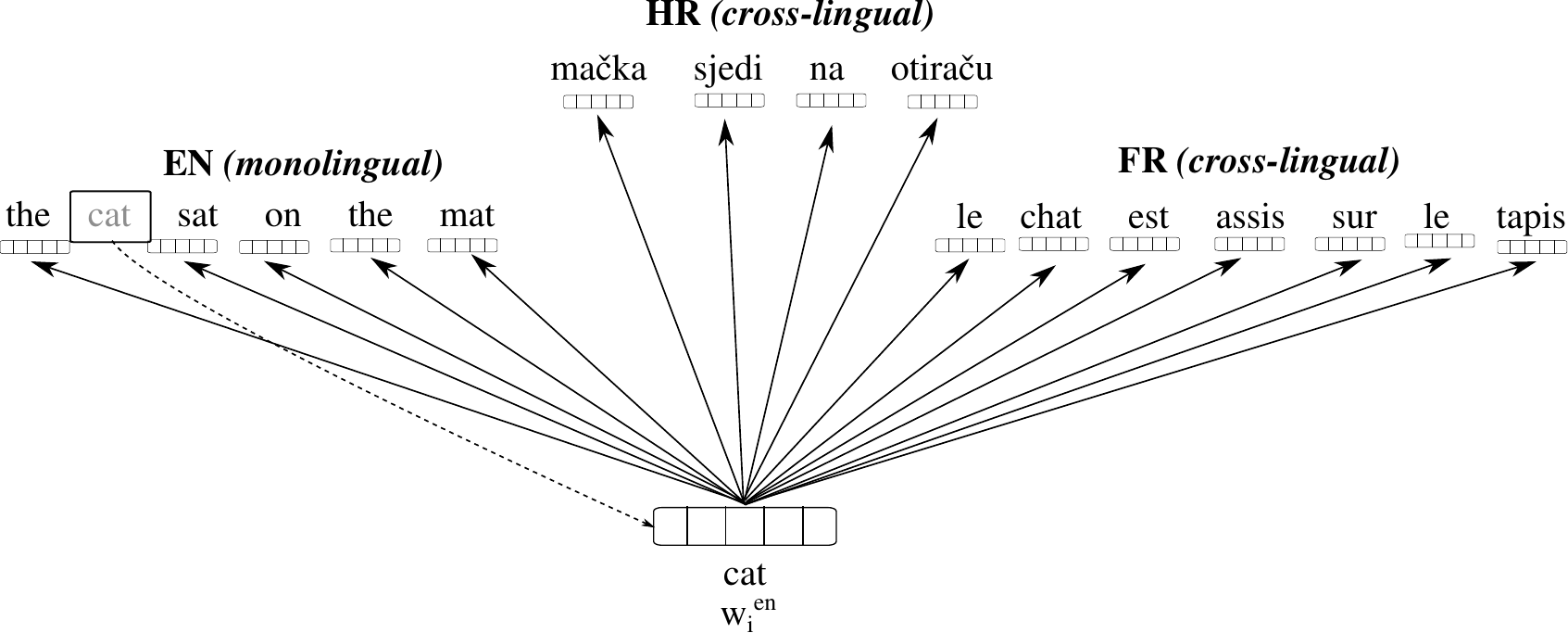}
        \caption{A multilingual extension of the sentence-level TransGram model of Coulmance et al. (2015). Since the model bypasses the word alignment step in its SGNS-style objective, for each center word (e.g., the EN word \textit{cat} in this example) the model predicts \textit{all} words in each sentence (from all other languages) which is aligned to the current sentence (e.g., the EN sentence \textit{the cat sat on the mat}).}
        \label{fig:multitransgram}
\end{figure*}

\paragraph{Pseudo-bilingual and joint approaches}
The two other sub-groups of word-level models also assume the existence of monolingual data plus multiple bilingual dictionaries covering translations of the same term in multiple languages. The main idea behind \textit{pseudo-multilingual} approaches is to ``corrupt'' monolingual data available for each of the $L$ languages so that words from all languages are present as context words for every center word in all monolingual corpora. A standard monolingual word embedding model (e.g., CBOW or SGNS) is then used to induce a shared multilingual space. First, for each word in each vocabulary one collects all translations of that word in all other languages. The sets of translations may be incomplete as they are dependent on their presence in input dictionaries. Following that, we use all monolingual corpora in all $L$ languages and proceed as the original model of \citeauthor{Gouws2015a} (2015): (i) for each word $w$ for which there is a set of translations of size $k_{mt}$, we flip a coin and decide whether to retain the original word $w$ or to substitute it with one of its $k_{mt}$ translations; (ii) in case we have to perform a substitution, we choose one translation as a substitute with a probability of $\frac{1}{2k_{mt}}$. In the limit, this method yields ``hybrid'' pseudo-multilingual sentences with each word surrounded by words from different languages. Despite its obvious simplicity, the only work that generalizes pseudo-bilingual approaches to the multilingual setting that we are aware of is by \citeauthor{Ammar2016} \citeyear{Ammar2016} who replace all tokens in monolingual corpora with their corresponding translation cluster ID, thereby restricting them to have the same representation. Note again that we do not need lexicons for all language pairs in case one resource-rich language (e.g., English) is selected as the pivot language.

Joint multilingual models rely on exactly the same input data (i.e., monolingual data plus multiple bilingual dictionaries) and the core idea is again to exploit multilingual word contexts. An extension of the \textit{joint} modeling paradigm to multilingual settings, illustrated in Figure~\ref{fig:multicbow}, is discussed by Duong et al \citeyear{Duong:2017eacl}. The core model is an extension of their bilingual model \cite{Duong2016} based on the CBOW-style objective: in the multilingual scenario with $L$ languages, for each language $l^i$ the training procedure consists of predicting the center word in language $l^i$ given the monolingual context in $l^i$ plus predicting all translations of the center word in all other languages, subject to their presence in the input bilingual dictionaries. Note that effectively this MultiCBOW model may be seen as a combination of multiple monolingual and cross-lingual CBOW-style sub-objectives as follows:

{\small
\begin{align}
J = \mathcal{L}_{\text{CBOW}}^{1} + \mathcal{L}_{\text{CBOW}}^2 + \cdots + \mathcal{L}_{\text{CBOW}}^L + \mathcal{L}_{\text{CBOW}}^{1 \rightarrow 2} + \mathcal{L}_{\text{CBOW}}^{2 \rightarrow 1} + \cdots \mathcal{L}_{\text{CBOW}}^{(L-1) \rightarrow L} + \mathcal{L}_{\text{CBOW}}^{L \rightarrow (L-1)} 
\label{eq:multicbow}
\end{align}}%
where the cross-lingual part of the objective again serves as the cross-lingual regularizer $\Omega$. By replacing the CBOW-style objective with the SGNS objective, the model described by Equation~\eqref{eq:multicbow} effectively gets transformed to the straightforward multilingual extension of the bilingual BiSkip model \cite{Luong2015b}. Exactly this model, called MultiSkip, is described in the work of Ammar et al. \citeyear{Ammar2016}. Instead of summing contexts around the center word as in CBOW, the model now tries to predict surrounding context words of the center word in its own language, plus its translations and all surrounding context words of its translations in all other languages. Translations are again obtained from input dictionaries or extracted from word alignments as in the original BiSkip and MultiSkip models. The pivot language idea is also applicable with the MultiSkip and MultiCBOW models.

\subsection{Multilingual Word Embeddings from Sentence-Level and Document-Level Information}
Extending bilingual embedding models which learn on the basis of aligned sentences and documents closely follows the principles already established for word-level models in Section~\ref{ss:multi_word_level}. For instance, the multilingual extension of the TransGram model from Coulmance et al. \citeyear{Coulmance2015} may be seen as MultiSkip without word alignment information (see again Table~\ref{tab:bilingual_skip_gram_comparison}). In other words, instead of predicting only words in the neighborhood of the word aligned to the center word, TransGram predicts all words in the sentences aligned to the sentence which contains the current center word (i.e., the model assumes uniform word alignment). This idea is illustrated by Figure~\ref{fig:multitransgram}. English is again used as the pivot language to reduce bilingual data requirements. 

The same idea of pivoting, that is, learning multiple bilingual spaces linked through the shared pivot English space is directly applicable to other prominent bilingual word embedding models such as \cite{Chandar2014}, \cite{Hermann2014}, \cite{Soyer2015}, \cite{Zou2013}, \cite{Gouws2015}. \citeauthor{Artetxe2018e} \citeyear{Artetxe2018e} use both English and Spanish as pivot languages for learning sentence embeddings of 93 languages.

The document-level model of \citeauthor{Vulic2016} (2016) may be extended to the multilingual setting using the same idea as in previously discussed word-level pseudo-multilingual approaches. \citeauthor{Sogaard2015} (2015) and \citeauthor{Levy2017} (2017) exploit the structure of a multi-comparable Wikipedia dataset and a multi-parallel Bible dataset respectively to directly build sparse cross-lingual representations using the same set of shared indices (i.e., the former uses the indices of aligned Wikipedia articles while the latter relies on the indices of aligned sentences in the multi-parallel corpus). Dense word embeddings are then obtained by factorizing the multilingual word-concept matrix containing all words from all vocabularies.

%% file: 10-Evaluation.tex
\section{Evaluation} \label{sec:evaluation}

Given the wide array of cross-lingual models and their potential applications, there is an equally wide array of possible evaluation settings. Cross-lingual embeddings can be evaluated on a set of \emph{intrinsic} and \emph{extrinsic} tasks that directly measure the quality of the embeddings. In addition, as we detail in Section \ref{sec:applications}, cross-lingual embeddings can be employed to facilitate cross-lingual transfer in any of a wide array of downstream applications.

In the following, we discuss the most common intrinsic and extrinsic evaluation tasks that have been used to test cross-lingual embeddings and outline associated challenges. In addition, we also present benchmarks that can be used for evaluating cross-lingual embeddings and the most important findings of two recent empirical benchmark studies.

\subsection{Intrinsic Tasks}

The first two widely used tasks are \emph{intrinsic} evaluation tasks: word similarity and multiQVEC(+). These tasks evaluate cross-lingual embeddings in a controlled \emph{in vitro} setting that is geared towards revealing certain characteristics of the representations. The major downside with these tasks is that good performance on them does not generalize necessarily to good performance on downstream tasks \cite{Tsvetkov2015b,Schnabel:2015emnlp}.

\paragraph{Word similarity} This task evaluates how well the notion of word similarity according to humans is emulated in the vector space. Multi-lingual word similarity datasets are multilingual extensions of datasets that have been used for evaluating English word representations. Many of these originate from psychology research and consist of word pairs -- ranging from synonyms (e.g., car - automobile) to unrelated terms (e.g., noon - string) --  that have been annotated with a relatedness score by human subjects. The most commonly used ones of these human judgement datasets are: a) the RG dataset \cite{rubenstein1965contextual}; b) the MC dataset \cite{miller1991contextual}; c) the WordSim-353 dataset \cite{Finkelstein:2002}, a superset of MC; and d) the SimLex-999 dataset \cite{hill2016simlex}. Extending them to the multilingual setting then mainly involves translating the word pairs into different languages: WordSim-353 has been translated to Spanish, Romanian, and Arabic \cite{Hassan:2009emnlp} and to German, Italian, and Russian \cite{leviant2015judgment}; RG was translated to German \cite{gurevych2005using}, French, \cite{joubarne2011comparison}, Spanish and Farsi \cite{camacho2015framework}; and SimLex-999 was translated to German, Italian and Russian \cite{leviant2015judgment} and to Hebrew and Croatian (Mrkšić et al., 2017b).
Other prominent datasets for word embedding evaluation such as MEN \cite{bruni2014multimodal}, RareWords \cite{luong2013better}, and SimVerb-3500 \cite{Gerz:2016emnlp} have only been used in monolingual contexts.

The SemEval 2017 task on cross-lingual and multilingual word similarity \cite{Camacho:2017semeval} has introduced cross-lingual word similarity datasets between five languages: English, German, Italian, Spanish, and Farsi, yielding 10 new datasets in total. Each cross-lingual dataset is of reasonable size, containing between 888 and 978 word pairs. 

Cross-lingual embeddings are evaluated on these datasets by first computing the cosine similarity of the representations of the cross-lingual word pairs. The Spearman's rank correlation coefficient \cite{myers2010research} is then computed between the cosine similarity score and the human judgement scores for every word pair. Cross-lingual word similarity datasets are affected by the same problems as their monolingual variants \cite{faruqui2016problems}: the annotated notion of word similarity is subjective and is often confused with relatedness; the datasets evaluate semantic rather than task-specific similarity, which is arguably more useful; they do not have standardized splits; they correlate only weakly with performance on downstream tasks; past models do not use statistical significance; and they do not account for polysemy, which is even more important in the cross-lingual setting. 

\paragraph{multiQVEC/multiQVEC+} multiQVEC+ is a multilingual extension of QVEC \cite{tsvetkov2015evaluation}, a method that seeks to quantify the linguistic content of word embeddings by maximizing the correlation with a manually-annotated linguistic resource. A semantic linguistic matrix is first constructed from a semantic database. The word embedding matrix is then aligned with the linguistic matrix and the correlation is measured using cumulative dimension-wise correlation. \citeauthor{Ammar2016} \citeyear{Ammar2016} propose QVEC+, which computes correlation using CCA and extend QVEC to the multilingual setting (multiQVEC) by using supersense tag annotations in multiple languages to construct the linguistic matrix. While QVEC has been shown to correlate well with certain semantic downstream tasks, as an intrinsic evaluation task it can only approximately capture the performance as it relates to downstream tasks.

\subsection{Extrinsic Tasks}

The two following tasks, word alignment and bilingual dictionary induction, are {\em extrinsic}~in the sense that they measure performance on tasks that are potentially of real-world importance. The tasks, at the same time, both evaluate cross-lingual word embeddings in a relatively direct manner, as they rely on nearest neighbor search in the cross-lingual embedding space to identify the most similar target word given a source word.

\paragraph{Word alignment prediction} For word alignment prediction, each word in a given source language sentence is aligned with the most similar target language word from the target language sentence. If a source language word is out of vocabulary, it is not aligned with anything, whereas target language out-of-vocabulary words are given a default minimum similarity score, and never aligned to any candidate source language word in practice \cite{Levy2017}. The inverse of the alignment error rate (1-AER) \cite{koehn2009statistical} is typically used to measure performance, where higher scores mean better alignments. \citeauthor{Levy2017} (2017) use alignment data from Hansards\footnote{\url{https://www.isi.edu/natural-language/download/hansard/}} and from four other sources \cite{graca2008building,lambert2005guidelines,mihalcea2003evaluation,holmqvist2011gold}.

\paragraph{Bilingual dictionary induction} 
After the shared cross-lingual embedding space is induced, given a list of $N$ source language words $x_{u,1},\ldots,x_{u,N}$, the task is to find a target language word $t$ for each \textit{query word} $x_u$ relying on the representations in the space. $t_i$ is the target language word closest to the source language word $x_{u,i}$ in the induced cross-lingual space, also known as the {\em cross-lingual nearest neighbor}. The set of learned $N$ $(x_{u,i},t_i)$ pairs is then run against a gold standard dictionary.

Bilingual dictionary induction is appealing as an evaluation task, as high-quality, freely available, wide-coverage manually constructed dictionaries are still rare, especially for non-European languages. The task also provides initial intrinsic evidence on the quality of the shared space. \citeauthor{Upadhyay2016} (2016) obtain evaluation sets for the task across 26 languages from the Open Multilingual WordNet \cite{bond2013linking}, while \citeauthor{Levy2017} (2017) obtain bilingual dictionaries from Wiktionary for Arabic, Finnish, Hebrew, Hungarian, and Turkish. More recently \citeauthor{Wijaya:2017emnlp} (2017) build evaluation data for 28 language pairs (where English is always the target language) by semi-automatically translating all Wikipedia words with frequency above 100. Most previous work \cite{Vulic2013a,Gouws2015,Mikolov2013b} filters source and target words based on part-of-speech, though this simplifies the task and introduces bias in the evaluation. Each cross-lingual embedding model is then evaluated on its ability to select the closest target language word to a given source language word as the translation of choice and measured based on precision-at-one (P@1). A more lenient evaluation measure is precision-at-k (P@k), where the correct translation must be retrieved in the ranked list of top $k$ target language words, $k\geq1$.\footnote{Typical values for $k$ are 1, 5, or 10.}


\section{Applications} \label{sec:applications}

\paragraph{Cross-lingual transfer} Both word alignment prediction and bilingual dictionary induction rely on (constrained) nearest neighbor search in the cross-lingual word embedding graph based on computed similarity scores.  However, cross-lingual word embeddings can also be used directly as features in NLP models. Such models are then defined for several languages, and can be used to facilitate \textit{cross-lingual transfer}. In other words, the main idea is to train a model on data from one language and then to apply it to another relying on shared cross-lingual features. Extrinsic evaluation on such downstream tasks is often preferred, as it directly allows to evaluate the usefulness of the cross-lingual embedding model for the respective task. We briefly describe the cross-lingual tasks that people have used to evaluate cross-lingual embeddings:

\begin{itemize}
\item \textit{Document classification} is the task of classifying documents with respect to topic, sentiment, relevance, etc. The task is commonly used following the setup of \citeauthor{Klementiev2012} (2012): it uses the RCV2 Reuters multilingual corpus\footnote{\url{http://trec.nist.gov/data/reuters/reuters.html}}. A document classifier is trained to predict topics on the document representations derived from word embeddings in the source language and then tested on the documents of the target language. Such representations typically do not take word order into account, and the standard embedding-based representation is to represent documents by the TF-IDF weighted average over the embeddings of the individual words, with an averaged perceptron model (or some other standard off-the-shelf classification model) acting as the document classifier. Word embeddings can also be used to seed more sophisticated classifiers based on convolutional or recurrent neural networks \cite{Liu:2016ijcai,Mandelbaum:2016arxiv,Zhang:2016naacl}. Although it is clear that cross-lingual word embeddings are instrumental to cross-lingual document classification, the task might be considered suboptimal for a full-fledged extrinsic evaluation of embeddings. It only evaluates topical associations and provides a signal for sets of co-occurring words, not for the individual words.
\item {\it Dependency parsing} is the task that constructs the grammatical structure of a sentence, establishing typed relationships between ``head'' words and words which modify those heads. In a cross-lingual setting \citeauthor{Tackstrom:2012naacl} (2012) proposed a parser transfer model that employed cross-lingual similarity measures based on cross-lingual Brown clusters. When relying on cross-lingual word embeddings, similar to cross-lingual document classification, a dependency parsing model is trained using the embeddings for a source language and is then evaluated on a target language. In the setup of \citeauthor{Guo2015} (2015), a transition-based dependency parser with a non-linear activation function is trained on Universal Dependencies data \cite{mcdonald2013universal}, with the source-side embeddings as lexical features\footnote{\url{https://github.com/jiangfeng1124/acl15-clnndep}}.
\item {\it POS tagging}, the task of assigning parts-of-speech to words, is usually evaluated using the Universal Dependencies treebanks (Nivre et al., 2016a)
as these are annotated with the same universal tag set. \citeauthor{Zhang2016} \citeyear{Zhang2016} furthermore map proper nouns to nouns and symbol makers (e.g. ``-'', ``/'') and interjections to an X tag as it is hard and unnecessary to disambiguate them in a low-resource setting. \citeauthor{Fang2017} (2017) use data from the CoNLL-X datasets of European languages \cite{buchholz2006conll}, from CoNLL 2003\footnote{\url{http://www.cnts.ua.ac.be/conll2003/ner/}} and from \citeauthor{das2011unsupervised} (2011), the latter of which is also used by \citeauthor{Gouws2015a} (2015). 
\item {\it Named entity recognition (NER)} is the task of tagging entities with their appropriate type in a text. \citeauthor{Zou2013} (2013) perform NER experiments for English and Chinese on OntoNotes \cite{hovy2006ontonotes}, while \citeauthor{Murthy2016} (2016) use English data from CoNLL 2003 \cite{tjong2003introduction} and Spanish and Dutch data from CoNLL 2002 \cite{sang2002introduction}.
\item {\it Super-sense tagging} is the task that involves annotating each significant entity in a text (e.g., nouns, verbs, adjectives and adverbs) within a general semantic taxonomy defined by the WordNet lexicographer classes (called super-senses). The cross-lingual variant of the task is used by \citeauthor{Gouws2015a} (2015) for evaluating their embeddings. They use the English data from SemCor\footnote{\url{http://web.eecs.umich.edu/mihalcea/
downloads.html#semcor}} and publicly available Danish data\footnote{\url{https://github.com/coastalcph/noda2015_sst}}.
\item {\it Semantic parsing} is the task of automatically identifying semantically salient targets in the text. Frame-semantic parsing, in particular, disambiguates the targets by assigning a sense (frame) to them, identifies their arguments, and labels the arguments with appropriate roles. \citeauthor{johannsen2015any} (2015) create a frame-semantic parsing corpus that covers five topics, two domains (Wikipedia and Twitter), and nine languages and use it to evaluate cross-lingual word embeddings.
\item {\it Discourse parsing} is the task of segmenting text into elementary discourse units (mostly clauses), which are then recursively connected via discourse relations to form complex discourse units. The segmentation is usually done according to Rhetorical Structure Theory (RST) \cite{mann1988rhetorical}. \citeauthor{Braud2017b} (2017) and \citeauthor{Braud2017} (2017) perform experiments using a diverse range of RST discourse treebanks for English, Portuguese, Spanish, German, Dutch, and Basque.
\item {\it Dialog state tracking (DST)} is the component in task-oriented dialogue statistical systems that keeps track of the belief state, that is, the system's internal distribution over the possible states of the dialogue. A recent state-of-the-art DST model of \citeauthor{Mrksic:2017acl} (2017a) is based exclusively on word embeddings fed into the model as its input. This property of the model enables a straightforward adaptation to cross-lingual settings by simply replacing input monolingual word embeddings with cross-lingual embeddings. Still an under-explored task, we believe that DST serves as a useful proxy task which shows the capability of induced word embeddings to support more complex language understanding tasks. \citeauthor{Mrksic:2017tacl} (2017b) use DST for evaluating cross-lingual embeddings on the Multilingual WOZ 2.0 dataset \cite{wen2016network} available in English, German, and Italian. Their results suggest that cross-lingual word embeddings boost the construction of dialog state trackers in German and Italian even without any German and Italian training data, as the model is able to also exploit English training data through the embedding space. Further, a multilingual DST model which uses training data from all three languages combined with a multilingual embedding space improves tracking performance in all three languages.

\item{\it Entity linking or wikification} is another task tackled using cross-lingual word embeddings \cite{Tsai:2016naacl}. The purpose of the task is to ground mentions written in non-English documents to entries in the English Wikipedia, facilitating the exploration and understanding of foreign texts without full-fledged translation systems \cite{ji:2015tac}. Such wikifiers, i.e., entity linkers are a valuable component of several NLP and IR tasks across different domains \cite{Mihalcea:2007cikm,Cheng:2013emnlp}.

\item {\it Sentiment analysis} is the task of determining the sentiment polarity (e.g. positive and negative) of a text. \citeauthor{Mogadala2016} (2016) evaluate their embeddings on the multilingual Amazon product review dataset of \citeauthor{prettenhofer2010cross} (2010). 

\item {\it Machine translation} is used to translate entire texts in other languages. This is in contrast to bilingual dictionary induction, which focuses on the translation of individual words. \citeauthor{Zou2013} (2013) used phrase-based machine translation to evaluate their embeddings. Cross-lingual embeddings are incorporated in the phrase-based MT system by adding them as a feature to bilingual phrase-pairs. For each phrase, its word embeddings are averaged to obtain a feature vector.

\item {\it Natural language inference} is the task of identifying whether a premise and a hypothesis entail, contradict, or are neutral towards each other. XNLI \cite{Conneau2018b} is a cross-lingual extension of MNLI \cite{Williams2017a}, which translates examples from the original corpus into ten different languages by employing professional translators. 

\end{itemize}

\paragraph{Information retrieval}
Word embeddings in general and cross-lingual word embeddings in specific have naturally found application beyond core NLP applications. They also offer support to Information Retrieval tasks (IR) \cite[inter alia]{Zamani:2016ictir,Mitra:2017arxiv} serving as useful features which can link semantics of the query to semantics of the target document collection, even when query terms are not explicitly mentioned in the relevant documents (e.g., the query can talk about \textit{cars} while a relevant document may contain a near-synonym \textit{automobile}). A shared cross-lingual embedding space provides means to more general cross-lingual and multilingual IR models without any substantial change in the algorithmic design of the retrieval process \cite{Vulic:2015sigir}. Semantic similarity between query and document representations, obtained through the composition process as in the document classification task, is computed in the shared space, irrespective of their actual languages: the similarity score may be used as a measure of document relevance to the information need formulated in the issued query.

\paragraph{Multi-modal and cognitive approaches to evaluation}

Evaluation of monolingual word embeddings is a controversial topic. Monolingual word embeddings are useful downstream \cite{turian2010word}, but in order to argue that one set of embeddings is better than another, we would like a robust evaluation metric. Metrics have been proposed based on co-occurrences (perplexity or word error rate), based on ability to discriminate between contexts (e.g., topic classification), and based on lexical semantics (predicting links in lexical knowledge bases). \citeauthor{sogaard2016evaluating} (2016) argues that such metrics are not valid, because co-occurrences, contexts, and lexical knowledge bases are also used to induce word embeddings, and that downstream evaluation is the best way to evaluate word embeddings. The only task-independent evaluation of embeddings that is reasonable, he claims, is to evaluate word embeddings by how well they predict behavioral observations, e.g. gaze or fMRI data. 

For cross-lingual word embeddings, it is easier to come up with valid metrics, e.g., Precision@k (P@$k$) in word alignment and bilingual dictionary induction. Note that these metrics only evaluate cross-lingual neighbors, not whether monolingual distances between words reflect synonymy relations. In other words, a random pairing of translation equivalents in vector space would score perfect precision in bilingual dictionary induction tasks. In addition, if we intend to evaluate the ability of cross-lingual word embeddings to allow for generalizations {\em within}~languages, we inherit the problem of finding valid metrics from monolingual word representation learning. 

\subsection{Benchmarks}

\paragraph{Benchmarks} In light of the plethora of both intrinsic and extrinsic evaluation tasks and datasets, a rigorous evaluation of cross-lingual embeddings across many benchmark datasets can often be cumbersome and practically infeasible. Existing benchmarks that evaluate on multiple tasks proposed previously \cite{Faruqui2014b,Ammar2016} are mostly no longer available. Available libraries for evaluation on bilingual lexicon induction are the MUSE \cite{Conneau2018}\footnote{\url{https://github.com/facebookresearch/MUSE}} and VecMap \cite{Artetxe2018b}\footnote{\url{https://github.com/artetxem/vecmap}} projects respectively. The recent xling-eval benchmark by \cite{Glavas2019}\footnote{\url{https://github.com/codogogo/xling-eval}} includes bilingual lexicon induction as well as three downstream tasks: cross-lingual document classification, natural language inference, and information retrieval. As a good practice, we generally recommend to evaluate cross-lingual word embeddings on an intrinsic task that is cheap to compute and on at least one downstream NLP task besides document classification. 

\paragraph{Benchmark studies} To conclude this section, we summarize the findings of three recent benchmark studies of cross-lingual embeddings: \citeauthor{Upadhyay2016} (2016) evaluate cross-lingual embedding models that require different forms of supervision on various tasks. They find that on word similarity datasets, models with cheaper supervision (sentence-aligned and document-aligned data) are almost as good as models with more expensive supervision in the form of word alignments. For cross-lingual classification and bilingual dictionary induction, more informative supervision is more beneficial: word-alignment and sentence-level models score better. Finally, for dependency parsing, models with word-level alignment are able to capture syntax more accurately and thus perform better overall. The findings by \citeauthor{Upadhyay2016} strengthen our hypothesis that the choice of the data is more important than the algorithm learning from the same data source. 

\citeauthor{Levy2017} (2017) evaluate cross-lingual word embedding models on bilingual dictionary induction and word alignment. In a similar vein as our typology that is based on the type and level of alignment, they argue that whether or not an algorithm uses a particular feature set is more important than the choice of the algorithm. In their experiments, they achieve the best results using sentence IDs as features to represent words, which outperforms using word-level source and target co-occurrence information. These findings lend further evidence and credibility to our typology that is based on the data requirements of cross-lingual embedding models. Models that learn from word-level and sentence-level information typically outperform other approaches, especially for finer-grained tasks such as bilingual dictionary induction. These studies furthermore raise awareness that we should not only focus on developing better cross-lingual embedding models, but also work on unlocking new data sources and new ways to leverage comparable data, particularly for languages and domains with only limited amounts of parallel training data.

Most recently, \citeauthor{Glavas2019} (2019) evaluate several state-of-the-art methods on a large number of language pairs on bilingual lexicon induction and on downstream tasks. They demonstrate that the performance of cross-lingual word embedding models often depends on the task they are applied to. For orthogonal mapping-based approaches, the performance on bilingual lexicon induction almost perfectly correlates with downstream performance on cross-lingual natural language inference and information retrieval, while for non-orthogonal methods, performance only weakly correlates. This makes downstream evaluation of such methods even more important. Finally, they find that methods that directly optimize for BLI such as \cite{Joulin2018a} exhibit reduced performance on downstream tasks.

%% file: 11-Challenges.tex
\section{General Challenges and Future Directions} \label{sec:challenges}

\paragraph{Subword-level information} In morphologically rich languages, words can have complex internal structures, and some word forms can be rare. For such languages, it makes sense to compose representations from representations of lemmas and morphemes. Neural network models increasingly leverage subword-level information \cite{sennrich2015neural,lample2016neural} and character-based input has been found useful for sharing knowledge in multilingual scenarios \cite{Gillick2016,Ling2016b}. Subword-level information has also been used for learning word representations \cite{Ling2015a,Bhatia2016} but has so far not been incorporated in learning cross-lingual word representations.

\paragraph{Multi-word expressions} Just like words can be too coarse units for representation learning in morphologically rich languages, words also combine in non-compositional ways to form multi-word expressions such as {\em ad hoc} or {\em kick the bucket}, the meaning of which cannot be derived from standard representations of their constituents. Dealing with multi-word expressions remains a challenge for monolingual applications and has only received scarce attention in the cross-lingual setting.



\paragraph{Function words} Models for learning cross-linguistic representations share weaknesses with other vector space models of language: While they are very good at modeling the conceptual aspect of meaning evaluated in word similarity tasks, they fail to properly model the functional aspect of meaning, e.g. to distinguish whether one remarks ``Give me \emph{a} pencil" or ``Give me \emph{that} pencil". Modeling the functional aspect of language is of particular importance in scenarios such as dialogue, where the pragmatics of language must be taken into account.

\paragraph{Polysemy} While conflating multiple senses of a word is already problematic for learning monolingual word representations, this issue is amplified in a cross-lingual embedding space: If polysemy leads to $m$ bad word embeddings in the source language, and $n$ bad word embeddings in the target language, we can derive $\mathcal{O}(n\times m)$ false nearest neighbors from our cross-lingual embeddings. While recent work on learning cross-lingual multi-sense embeddings \cite{li2015multi} is extremely interesting, it is still an open question whether modern NLP models can infer from context, what they need in order to resolve lexical ambiguities.


\paragraph{Embeddings for specialized domains} There are many domains, for which cross-lingual applications would be particularly useful, such as bioinformatics or social media. However, parallel data is scarce in many such domains as well as for low-resource languages. Creating robust cross-lingual word representations with as few parallel examples as possible is thus an important research avenue. An important related direction is to leverage comparable corpora, which are often more plentiful and incorporate other signals, such as from multi-modal contexts.

For many domains or tasks, we also might want to have not only word embeddings, but be able to compose those representations into accurate sentence and document representations. Besides existing methods that sum word embeddings, not much work has been doing on learning better higher-level cross-lingual representations.

\paragraph{Feasibility}


Learning a general shared vector space for words that reliably captures inter-language and intra-language relations may seem slightly optimistic. Languages are very different, and it is not clear if there is even a definition of {\em words} that make words commensurable across languages. Note that while this is related to whether it is possible to translate between the world's languages in the first place, the possibility of translation (at document level) does not necessarily entail that it is possible to device embeddings such that translation equivalents in two languages end up as nearest neighbors.

There is also the question of what is the computational complexity of finding an embedding that obeys all our inter-lingual and intra-lingual constraints, say, for example, translation equivalents and synonymy. Currently, many approaches to cross-lingual word embeddings, as shown in this survey, minimize a loss that penalizes models for violating such constraints, but there is no guarantee that the final model satisfies all constraints. 

Checking whether all such constraints are satisfied in a given model is trivially done in time linear in the number of constraints, but finding out whether such a model exists is much harder. While the problem's decidability follows from the decidability of two-variable first order logic with equivalence/symmetry closure, determining whether such a graph exists is in fact NP-hard \cite{eades1995nearest}.



\paragraph{Non-linear mapping} Mapping-based approaches assume that a linear transformation can project the embedding space of one language into the space of a target language. While \citeauthor{Mikolov2013b} (2013b) and \citeauthor{Conneau2018} (2018a) both find that a linear transformation outperforms non-linear transformation learned via a feedforward neural network, assuming a linear transformation between two languages is overly simplistic and ignores important language-specific differences. \citeauthor{Nakashole2018} \citeyear{Nakashole2018} lend further credibility to this intuition by learning neighbourhood-specific linear transformations and showing that these vary across the monolingual word embedding space. However, to the best of our knowledge, there has not been any model yet that leveraged this intuition to construct a more effective mapping model.

\paragraph{Robust unsupervised approaches} Recently, word-level mapping-based approaches have become the preferred choice for learning cross-lingual embeddings due to their ease of use and reliance on inexpensive forms of supervision. At the same time, methods for learning with less supervision have been developed: These range from approaches using small seed lexicons \cite{Zhang2016l,Artetxe:2017acl} to completely unsupervised approaches that seek to match source and target distributions based on adversarial learning \cite{Zhang2017,Zhang2017j,Conneau2018} and offer support to neural machine translation and cross-lingual information retrieval from monolingual data only \cite{Lample:2018iclr,Artetxe:2018iclr,Litschko:2018sigir,Artetxe2018b}. Such unsupervised methods, however, rely on the assumption that monolingual word embedding spaces are approximately isomorphic, which has been shown not to hold in general and for distant language pairs in particular \cite{Soegaard2018}, for which such methods are desired in the first place. In simple terms, although thought-provoking and attractive in theory, such unsupervised methods thus fail when languages are distant. In such cases, using a distantly supervised seed lexicon of identical strings in both languages is often preferable \cite{Soegaard2018}. Based on a recent large-scale evaluation of cross-lingual word embedding models on bilingual lexicon induction and downstream tasks \cite{Glavas2019}, the model of \citeauthor{Artetxe2018b} \citeyear{Artetxe2018b} is currently the most robust among completely unsupervised approaches.